\documentclass{IEEEtran}
\usepackage{cite}

\newtheorem{proposition}{Proposition}
\newtheorem{definition}{Definition}
\newtheorem{proof}{Proof}
\newtheorem{example}{Example}
\usepackage{booktabs}
\usepackage{mathtools}
\usepackage{tikz}
\usepackage{tikz-cd}
\usetikzlibrary{arrows.meta,shapes,positioning,calc}
\usetikzlibrary{fit, positioning, shapes.geometric}

\usepackage{verbatim}

\usepackage[labelfont={color=accessblue,bf,sf},textfont={}]{caption}
\usepackage[labelfont={},textfont={}]{subcaption}
\usepackage{enumitem}
\usepackage{amsmath,amssymb,amsfonts}
\usepackage{algorithmic}
\usepackage{graphicx}
\usepackage{textcomp}
\usepackage{bm}
\usepackage{hyperref}
\usepackage{cleveref}
\usepackage{float}
\usepackage[normalem]{ulem}
\usepackage{booktabs}
\usepackage{graphicx}
\usepackage{adjustbox}
\usepackage{soul}
\usepackage{xcolor}
\soulregister{\cite}{1}
\definecolor{accessblue}{RGB}{0,105,154}
\makeatletter
\def\ps@IEEEtitlepagestyle{%
  \def\@oddfoot{\mycopyrightnotice}%
  \def\@oddhead{\hbox{}\@IEEEheaderstyle\leftmark\hfil\thepage}\relax
  \def\@evenhead{\@IEEEheaderstyle\thepage\hfil\leftmark\hbox{}}\relax
  \def\@evenfoot{}%
}
\def\mycopyrightnotice{%
  \begin{minipage}{\textwidth}
  \centering \scriptsize
 \textcopyright~2025 IEEE. Personal use of this material is permitted. Permission from IEEE must be obtained for all other uses, in any current or future media, including reprinting/republishing this material for advertising or promotional purposes, creating new collective works, for resale or redistribution to servers or lists, or reuse of any copyrighted component of this work in other works. \\
Published in: IEEE Access, vol. 13, pp. 131823-131838, 2025. DOI: 10.1109/ACCESS.2025.3592104.
  \end{minipage}
}
\makeatother

\def\BibTeX{{\rm B\kern-.05em{\sc i\kern-.025em b}\kern-.08em
    T\kern-.1667em\lower.7ex\hbox{E}\kern-.125emX}}

\begin{document}
\title{Hypergraph Neural Sheaf Diffusion: A Symmetric Simplicial Set Framework for Higher-Order Learning}

\author{
    Seongjin Choi\textsuperscript{1},
    Gahee Kim\textsuperscript{2},
    Yong-Geun Oh\textsuperscript{1*}\thanks{*Corresponding author: yongoh1@postech.ac.kr}%
    \thanks{Seongjin Choi and Yong-Geun Oh were supported by the IBS project \# IBS-R003-D1}%
    \\
    \IEEEauthorblockA{\textsuperscript{1}POSTECH, Gyeongbuk, Korea \& Center for Geometry and Physics, Institute for Basic Science (IBS), 79 Jigok-ro 127beon-gil, Nam-gu, Pohang, Gyeongbuk, KOREA 37673}
    \\
    \IEEEauthorblockA{\textsuperscript{2}Kim Jaechul Graduate School of AI, Korea Advanced Institute of Science and Technology (KAIST), Seoul 02455, South Korea}
    \\
    \IEEEauthorblockA{\textsuperscript{1*}Center for Geometry and Physics, Institute for Basic Science (IBS),
 79 Jigok-ro 127beon-gil, Nam-gu, Pohang, Gyeongbuk, KOREA 37673 \& POSTECH, Gyeongbuk, Korea}
}
\maketitle

\begin{abstract}
The absence of intrinsic adjacency relations and orientation systems in hypergraphs creates fundamental challenges for constructing sheaf Laplacians of arbitrary degrees. We resolve these limitations through symmetric simplicial sets derived directly from hypergraphs, called symmetric simplicial lifting, which encode all possible oriented subrelations within each hyperedge as ordered tuples. This construction canonically defines adjacency via facet maps while inherently preserving hyperedge provenance. We establish that the normalized degree zero sheaf Laplacian on our symmetric simplicial lifting reduces exactly to the traditional graph normalized sheaf Laplacian when restricted to graphs, validating its mathematical consistency with prior graph-based sheaf theory. Furthermore, the induced structure preserves all structural information from the original hypergraph, ensuring that every multi-way relational detail is faithfully retained. Leveraging this framework, we introduce Hypergraph Neural Sheaf Diffusion (HNSD), the first principled extension of neural sheaf diffusion to hypergraphs. HNSD operates via normalized degree zero sheaf Laplacian over symmetric simplicial lifting, resolving orientation ambiguity and adjacency sparsity inherent to hypergraph learning. Experimental evaluations demonstrate HNSD’s competitive performance across established benchmarks.
\end{abstract}

\begin{IEEEkeywords}
Cellular sheaf theory, hypergraph Laplacian, hypergraph neural networks, sheaf Laplacian, symmetric simplicial sets 
\end{IEEEkeywords}

\maketitle
\section{Introduction}
\label{sec:introduction}

Hypergraphs are mathematical structures that effectively represent higher-order relationships among entities through hyperedges, each of which may connect more than two nodes simultaneously\cite{bretto2013hypergraph}. As many real-world datasets naturally form hypergraphs\cite{gao20123, yan2020learning, wang2022multitask}, constructing effective neural network architectures on hypergraphs has become crucial for learning from such complex data\cite{feng2019hypergraph, gao2020hypergraph, wu2022hypergraph}. A foundational approach for relational data involves designing neural architectures around Laplacian operators, which encode relational structures into learnable diffusion processes. 

Graphs—viewed as 1-dimensional hypergraphs where hyperedges connect exactly two nodes—provide a natural foundation for Laplacian-based architectures. Graphs satisfy two critical structural conditions: (1) adjacency between nodes is canonically defined through pairwise edges, and (2) no total ordering of nodes is required to define the Laplacian. This structural simplicity allows graph Laplacians to encode relational patterns effectively, forming the basis for graph convolutional networks (GCNs)\cite{kipf2017semi, defferrard2016convolutional}, which have shown significant success, especially for homophilic graph datasets\cite{xia2021graph, zhang2020deep, wu2020comprehensive, yuan2022explainability}. Cellular sheaf theory~\cite{curry2014sheaves} enhances this framework by equipping graphs with additional algebraic-topological structure, assigning vector spaces to nodes and edges, and defining consistency constraints between them. This leads to the degree 0 sheaf Laplacian~\cite{hansen2019toward}, a generalization of the standard graph Laplacian, which underpins the development of neural sheaf diffusion (NSD) models on graphs~\cite{hansen2020laplacians, bodnar2022neural, barbero2022sheaf}. Unlike traditional GCNs, NSD leverages the expressive power of cellular sheaves to address two major limitations: poor performance on heterophilic graphs and the oversmoothing problem, where node representations become indistinguishable in deep architectures.

However, hypergraphs pose inherent challenges due to their structural complexity, making it extremely difficult to define sheaf Laplacians of various degrees as is done for graphs. In graphs, which model only 1-dimensional relations, the degree 0 sheaf Laplacian suffices to capture all connectivity patterns. In contrast, hypergraphs naturally encode high-dimensional relations, necessitating the definition of sheaf Laplacians of arbitrary degrees to represent their structure fully. Achieving this requires a well-defined notion of adjacency between hyperedges. Yet, by design, hypergraphs lack explicit adjacency among hyperedges-even with many hyperedges present, their multi-node nature often leads to sparse or undefined adjacency patterns. This sparsity can trivialize Laplacian operators, rendering them ineffective for learning tasks. These limitations highlight a critical gap: while Laplacian-based architectures are central to relational learning, their direct application to hypergraphs remains hindered by structural mismatches, motivating novel approaches to reconcile hypergraph complexity with effective operator design.
\begin{figure*}
    \centering
    \begin{subfigure}[b]{0.3\textwidth}
        \centering
        \begin{tikzpicture}
        [
            he/.style={draw, ellipse, dotted, inner sep=10pt},        
            node/.style={circle, fill=black, inner sep=0pt, minimum size=1.2mm, label distance=2mm} 
        ]
        \node[node, label=above:{\footnotesize $v_0$}] (v0) at (-1.5,0.75) {};
        \node[node, label=below:{\footnotesize $v_1$}] (v1) at (-1.5,-0.1875) {};
        \node[node, label=below:{\footnotesize $v_2$}] (v2) at (0,0.1875) {};
        \node[node, label=above:{\footnotesize $v_3$}] (v3) at (0,1.5) {};
        \node[node, label=right:{\footnotesize $v_4$}] (v4) at (1.5,0.9375) {};
        \node[draw=none, fill=none] (g1) at (0.9375,-0.1875) {}; 
        \node[draw=none, fill=none] (g2) at (0,0.84) {}; 

        \node[label=left:{\footnotesize $e$}] [he, fit=(v0) (v1) (v2) (v3)] {};
        \node[label=right:{\footnotesize $e'$}] [he, fit=(v2) (v3) (v4) (g1)] {};
        \end{tikzpicture}
       \caption{Hypergraph.}
    \end{subfigure}
    \hspace{0.02\textwidth}%
    \begin{subfigure}[b]{0.3\textwidth}
        \centering
        \begin{tikzpicture}
        [
            he/.style={draw, ellipse, dotted, inner sep=10pt},        
            node/.style={circle, fill=black, inner sep=0pt, minimum size=1.2mm, label distance=2mm} 
        ]
        \node[node, label=left:{\footnotesize $v_0$}] (w0) at (4.5,0.75) {};
        \node[node, label=below:{\footnotesize $v_1$}] (w1) at (4.5,-0.1875) {};
        \node[node, label=below:{\footnotesize $v_2$}] (w2) at (6,0.1875) {};
        \node[node, label=above:{\footnotesize $v_3$}] (w3) at (6,1.5) {};
        \node[node, label=right:{\footnotesize $v_4$}] (w4) at (7.5,0.9375) {};
        \node[draw=none, fill=none] (a1) at (6.9375,-0.1875) {}; 
        \node[draw=none, fill=none] (a2) at (4.875,-0.09375) {};
        \node[draw=none, fill=none] (a3) at (5.625,0.09375) {};
        \fill[gray!20] (w0.center) -- (w1.center) -- (w2.center) -- (w3.center) -- cycle;
        \fill[gray!20] (w2.center) -- (w3.center) -- (w4.center) -- cycle;
        \node[he, fit=(w0) (w1) (w2) (w3)] {};
        \node[he, fit=(w2) (w3) (w4) (a1)] {};
 
        \draw (w0) -- (w1) -- (w2) -- (w3) -- (w0) -- cycle;
        \draw[dotted] (w0) -- (w2);
        \draw (w1) -- (a2) -- cycle;
        \draw (w1) -- (w3);
        \draw (w2) -- (w3) -- (w4) -- (w2) -- cycle;
        \end{tikzpicture}
        \caption{Induced simplicial complex.}
        \label{introinducedsimplicialcomplex}
    \end{subfigure}
    \begin{subfigure}[b]{0.3\textwidth}
        \centering
        \begin{tikzpicture}
         [
            he/.style={draw, ellipse, dotted, inner sep=10pt},        
            node/.style={circle, fill=black, inner sep=0pt, minimum size=1.2mm, label distance=2mm} 
        ]
        \node[node, label=left:{\footnotesize $v_0$}] (w0) at (4.5,0.75) {};
        \node[node, label=below:{\footnotesize $v_1$}] (w1) at (4.5,-0.1875) {};
        \node[node, label=below:{\footnotesize $v_2$}] (w2) at (6,0.1875) {};
        \node[node, label=above:{\footnotesize $v_3$}] (w3) at (6,1.5) {};
        \node[node, label=right:{\footnotesize $v_4$}] (w4) at (7.5,0.9375) {};
        \node[draw=none, fill=none] (a1) at (6.9375,-0.1875) {}; 
        \node[draw=none, fill=none] (a2) at (4.875,-0.09375) {};
        \node[draw=none, fill=none] (a3) at (5.625,0.09375) {};
        \node[he, fit=(w0) (w1) (w2) (w3)] {};
        \node[he, fit=(w2) (w3) (w4) (a1)] {};
        \fill[gray!20] (w0.center) -- (w1.center) -- (w2.center) to[bend left] (w3.center) -- cycle;
        \fill[gray!20] (w2.center) to[bend right] (w3.center) -- (w4.center) -- cycle;
        \draw[-latex] (w0) -- (w1);
        \draw[-latex] (w1) -- (w2);
        \draw[-latex] (w0) -- (w3);
        \draw[-latex, dotted] (w0) -- (w2);
        \draw (w1) -- (a2) -- cycle;
        \draw (w1) -- (w3);
        \draw[-latex] (w3) -- (w4);
        \draw[-latex] (w2) -- (w4);
        \draw[-latex, bend left] (w2) to node[midway, left, xshift=0.1cm] {} (w3);
        \draw[-latex, bend right] (w2) to node[midway, right, xshift=-0.1cm] {} (w3);

        \end{tikzpicture}
        \caption{Symmetric simplicial lifting.}
        \label{introinducedsymsset}
    \end{subfigure}
    \caption{(a) Hypergraph lacks hierarchical structure. (b) The induced simplicial complex loses contextual information. (c) Symmetric simplicial lifting preserves hyperedge-specific context.}
\end{figure*}

Given the structural challenges of hypergraphs, a natural strategy to define sheaf Laplacians involves embedding hypergraphs into simplicial complexes \cite{landry2024simpliciality} by adjoining all subsets of each hyperedge as simplices (Fig.~\ref{introinducedsimplicialcomplex}). This approach interprets simplex as a possible subrelation of each hyperedge represented by a set. The induced structure inherently supports adjacency: two simplices are adjacent if they share a common facet or cofacet, creating a dense network of relationships \cite{yang2022simplicial, wei2025persistent, russold2022persistent}. Combined with the ability to induce orientations via a total node order, simplicial complexes initially appear well-suited for defining sheaf Laplacians of various degrees. However, this strategy introduces two critical limitations. First, hyperedge-specific structural information is lost: $\{v_2, v_3\}$ in Fig.~\ref{introinducedsimplicialcomplex} become indistinguishable whether it represents possible subrelations of hyperedge $e$, hyperedge $e'$, or independently observed hyperedge. This ambiguity erases contextual relationships between subrelations within hyperedge and the original hyperedges, making it impossible to reconstruct the original hypergraph; and (2) dependence on an arbitrary total node order to define orientations~\cite{bodnar2021weisfeilertop}.  

In this work, we address these limitations by introducing a symmetric simplicial set explicitly constructed from hypergraphs. Our framework collects all possible tuples in each hyperedge with its originating hyperedge (Fig.~\ref{introinducedsymsset}). These tuples, which we term simplices, encode all possible oriented subrelations derived from the observed hyperedges while preserving their hyperedge-specific context. The induced structure, called symmetric simplicial lifting, inherently defines adjacency via facet maps and supports sheaf Laplacians of arbitrary degrees without requiring auxiliary node orderings. Crucially, our construction generalizes traditional graph-based sheaf theory: the normalized degree 0 sheaf Laplacian on symmetric simplicial lifting reduces exactly to the traditional normalized sheaf Laplacian when restricted to the graph. This equivalence validates our framework as a principled extension of sheaf theory to hypergraphs. Leveraging this foundation, we develop Hypergraph Neural Sheaf Diffusion (HNSD), which implements NSD on hypergraphs via normalized degree 0 sheaf Laplacian over symmetric simplicial lifting. Our key contributions can be summarized as follows:
\begin{itemize}
    \item We introduce a symmetric simplicial set construction from a hypergraph, called symmetric simplicial lifting, enabling the formulation of high-degree sheaf Laplacians on a hypergraph.
    \item We demonstrate two fundamental properties of symmetric simplicial lifting: it preserves all structural information of the original hypergraph, and the normalized degree 0 sheaf Laplacian is equal to the traditional normalized sheaf Laplacian in the case of graphs.
    \item We introduce new architecture HNSD, the first extension of NSD to hypergraphs, implemented via the normalized degree 0 sheaf Laplacian on symmetric simplicial lifting. HNSD resolves fundamental orientation and adjacency ambiguities in higher-order learning while achieving competitive performance across established hypergraph benchmarks.
\end{itemize}

The remainder of this paper is structured as follows. Section~\ref{sec:relatedworks} reviews related work across key areas, including sheaf learning on graphs, advancements in hypergraph networks, and the representation of hypergraphs using symmetric simplicial sets. Section~\ref{sec:preliminaries} provides the essential background on sheaf Laplacians on graphs and the fundamental concepts of hypergraphs. Building on this, Section~\ref{sec:sheaflaplaciansonhighorderdomain} introduces the advanced concepts of ordered simplicial complexes and symmetric simplicial sets, establishing the formulation of sheaf Laplacians on these higher-order domains. In Section~\ref{sec:methodology}, we construct the symmetric simplicial lifting of a hypergraph, establish some of its properties, and then propose the HNSD framework. Finally, experimental evaluations on real-world datasets are presented in Section~\ref{sec:experiments}, and concluding remarks are offered in Section~\ref{sec:conclusion}.

\begin{table}[h!]
\centering
\caption{Summary of symbols used in the paper.}
\label{table:example}
\resizebox{0.8\columnwidth}{!}{
\begin{tabular}{lc} 
    \toprule
    \textbf{Notation} & \textbf{Description} \\ 
    \midrule
    $\mathbb{N}$ & the set of non-negative integers \\ 
    $[n]$ & $\{0, 1, \cdots, n\}$ for $n \in \mathbb{N}$ \\ 
    $\mathfrak{S}_n$ & all bijections from $[n]$ to $[n]$ \\
    $\Vert A \Vert$ & $|A| -1$ for a set $A$ \\ 
    $2^B$ & the set of all subsets of $B$ \\ 
    $B^A$ & the set of all functions from $A$ to $B$ \\
    $\operatorname{Id}$ & the identity map on any set to itself \\
    \bottomrule
\end{tabular}
}
\end{table}

\section{Related work}
\label{sec:relatedworks}

\subsection{Sheaf learning on graphs}
The graph Laplacian serves as the foundational tool for constructing neural network architectures~\cite{defferrard2016convolutional, kipf2017semi}. Despite their remarkable success in homophilic graph learning, GCNs face significant challenges in heterophilic graphs~\cite{zhu2020beyond} and suffer from oversmoothing with increasing network depth~\cite{nt2019revisiting,oonograph}. These limitations stem from the fundamental constraints of traditional graph Laplacian-based diffusion processes. Recent advances addressed these issues through geometric reinterpretations using cellular sheaf theory~\cite{curry2014sheaves}. The foundational implementation emerged with Sheaf Neural Networks~\cite{hansen2020sheaf}, which replace standard graph Laplacians with sheaf Laplacians to enable more sophisticated feature transport. Subsequent developments introduced learnable cellular sheaves~\cite{bodnar2022neural,barbero2022sheaf}, demonstrating that principled diffusion processes guided by sheaf geometry can simultaneously mitigate oversmoothing and improve performance on heterophilic graphs. 

\subsection{Hypergraph networks}
In hypergraph learning, hypergraph Laplacians derived from incidence matrices have been used to construct neural networks on hypergraphs. HGNN~\cite{feng2019hypergraph} generalizes graph convolution by introducing a hypergraph Laplacian constructed from the incidence matrix, enabling each node to update its features through shared structures with other nodes in its hyperedges. HyperGCN~\cite{yadati2019hypergcn} improved computational efficiency by approximating hypergraph convolution through pairwise edges. HNHN~\cite{dong2020hnhn} applied nonlinear transformations to both nodes and hyperedges, while Hyper-Conv and Hyper-Atten~\cite{bai2021hypergraph} incorporated attention mechanisms to dynamically weight hypergraph components.

Subsequent works have further enhanced the expressiveness and adaptability of hypergraph neural networks. HCoN~\cite{wu2022hypergraph} explored collaborative learning between nodes and hyperedges using a reconstruction-based framework. Learnable hypergraph structures were introduced in~\cite{zhang2022learnable}, enabling adaptive optimization of hyperedge connections. HGNN+~\cite{gao2022hgnn} further generalized HGNN by learning adaptive hyperedge weights to better capture their relative importance. HCNH~\cite{wu2022hypergraphconvolution} improved node classification by jointly filtering node and hyperedge features while minimizing reconstruction loss. CCL~\cite{wu2024collaborative} enhanced hypergraph node classification by contrasting GCN and Hyper-Conv views in each layer, enabling collaborative representation learning.

Collectively, these studies have advanced hypergraph neural networks from early spectral formulations to adaptive, attention-based, and contrastive learning frameworks. These incidence matrix-derived operators have also been extended through sheaf-theoretic frameworks, where node-level analysis is enriched via sheaf Laplacians that model feature transport through hyperedge-mediated interactions~\cite{duta2023sheaf}. While simplicial complexes impose additional structural constraints compared to general hypergraphs, recent work has demonstrated their value in neural architectures through orientation-aware diffusion. Specifically, degree $n$ sheaf Laplacians have been adapted to operate on simplicial complexes~\cite{wei2025persistent,russold2022persistent,yang2022simplicial}, governing how features at $n$-dimensional simplices diffuse across adjacent structures.

\subsection{Symmetric Simplicial Sets from Hypergraphs} 
A categorical framework provides a functorial mapping from ordered hypergraphs to their canonical simplicial set representations~\cite{spivak2009higher}. This approach was later extended to general hypergraphs via an explicit functorial construction that generates symmetric simplicial sets~\cite{choi2024cellular}, thereby providing a concrete categorical mechanism to study hypergraph data through homological algebra.

\section{Preliminaries}
\label{sec:preliminaries}
In this section, we review the sheaf Laplacian on graphs and notions of hypergraphs.

\subsection{Sheaf Laplacian on graphs}
A graph is a mathematical structure representing pairwise (1-dimensional) relations among entities. Formally, a graph $G$ is a triple $(V(G), E(G), f_G)$, where $V(G)$ is the set of nodes, $E(G)$ is the set of edges, and $f_G: E(G) \to 2^{V(G)}$ is the labeling function assigning each edge to a subset of exactly two nodes. A node $v$ is said to be incident to an edge $e$, denoted by $v \in e$, if and only if $v \in f_G(e)$.    

In the context of graph-based data analysis, vector-valued features are typically assigned to nodes. A central question is how these node features diffuse along edges, and this diffusion is precisely captured by the notion of a cellular sheaf and sheaf Laplacian~\cite{curry2014sheaves, hansen2019toward}.\\

\begin{definition}
Given a graph $G$, a \textit{cellular sheaf} $(G, \mathcal{F})$ consists of the following data:
\begin{itemize}
    \item For each node $v \in V(G)$ and each edge $e \in E(G)$, real vector spaces $\mathcal{F}(v)$ and $\mathcal{F}(e)$, called the \textit{stalks} at node $v$ and edge $e$, respectively.
    \item For each $v \in e$, a linear map $\mathcal{F}(v \in e) : \mathcal{F}(v) \rightarrow \mathcal{F}(e)$.
\end{itemize}

A \textit{0-cochain} $\mathbf{x}$ on the graph $G$, denoted by $\mathbf{x} = (x_v)_{v \in V}$, is an element of the direct sum of stalks over all nodes of $G$. We denote the set of all 0-cochains as $C^0(G, \mathcal{F})$. For a given 0-cochain $\mathbf{x}$, the $v$-component of the \textit{sheaf Laplacian}, denoted by $L_{\mathcal{F}}(\mathbf{x})_v$, is defined as \\
\begin{equation}\label{sheafLaplacianongraph}
    L_{\mathcal{F}}(\mathbf{x})_v = \sum_{v,u \in e}\mathcal{F}^*(v \in e)\left(\mathcal{F}(v \in e)(x_v) - \mathcal{F}(u \in e)(x_u) \right).
\end{equation}
Additionally, the $v$-component of the diagonal blocks, denoted by $D_{\mathcal{F}}(\mathbf{x})_v$, is defined as
        \begin{equation}\label{diagonalblockofsheafLaplacianongraph}
            \sum_{\{e \mid v \in e \}}\mathcal{F}^*(v \in e)\mathcal{F}(v \in e)(x_{v}).
        \end{equation}
Finally, the $v$-component of the normalized sheaf Laplacian, denoted by $\mathcal{L}_{\mathcal{F}}(\mathbf{x})_v$, is defined by
        \begin{equation}\label{normalizedsheafLaplacianongraph}
            \left((D_{\mathcal{F}})^{-\frac{1}{2}}L_{\mathcal{F}}(D_{\mathcal{F}})^{-\frac{1}{2}}\right)(\mathbf{x})_{v}.
        \end{equation}
\end{definition}

A cellular sheaf $(G,\mathcal{F})$ enables the diffusion of a node feature $x_v$ to an adjacent node $w$ via $\mathcal{F}^*(w \in e)\mathcal{F}(v \in e)(x_v)$. A 0-cochain $\mathbf{x}$ represents the collection of node features $x_v$ for each node $v$. The operator $L_{\mathcal{F}}(\mathbf{x})$ describes how the collection of node features $\mathbf{x}$ diffuses to adjacent nodes according to the cellular sheaf $\mathcal{F}$, thereby generalizing the graph Laplacian. Similarly, $\mathcal{L}_{\mathcal{F}}$ extends the concept of the normalized graph Laplacian.

\subsection{Hypergraphs}
A hypergraph generalizes the notion of a graph by capturing not only pairwise relationships but also higher-order, multidimensional interactions among entities. \\

\begin{definition}A hypergraph $H = (V(H), E(H), f_H)$ consists of the following data:
\begin{enumerate}
    \item $V(H)$ is a set of nodes of $H$
    \item $E(H)$ is a set of hyperedges of $H$
    \item $f_H : E(H) \to 2^{V(H)} \backslash V(H)$, a labeling function of $H$. 
\end{enumerate}
For each hyperedge $e \in E(H)$, we denote its dimension by $\Vert e \Vert \coloneqq \Vert f_H(e) \Vert$. Two hypergraphs $H$ and $H'$ are isomorphic~\cite{bretto2013hypergraph} if there exist bijections $a : V(H) \to V(H')$ and $b: E(H) \to E(H')$ satisfying: 
\begin{equation}
a_* \circ f_H = f_{H'} \circ b, (a^{-1})_* \circ f_{H'} =f_H \circ b^{-1}    
\end{equation}
where the induced map $a_* : 2^{V(H)} \to 2^{V(H')}$ is defined as 
\begin{equation}
a_*(\{v_0, \cdots, v_n \}) \coloneqq \{a(v_0), \cdots, a(v_n) \}.    
\end{equation}
\end{definition}

A graph is a special case of a hypergraph in which every hyperedge has dimension one; equivalently, $\Vert e \Vert =1$ for all $e \in E(H)$. Two hypergraphs are isomorphic if the patterns of how entities participate together in high-order relations are the same in both, regardless of the names of the entities.

\section{Sheaf Laplacians on Higher-Order domains}
\label{sec:sheaflaplaciansonhighorderdomain}
In this section, we define ordered simplicial complexes and symmetric simplicial sets, both of which organize collections of oriented $n$-dimensional simplices. Next, we explain how adjacency between simplices is defined via facet maps, which we use to extend the sheaf Laplacian to both ordered simplicial complexes and symmetric simplicial sets.

\subsection{Ordered simplicial complexes}
A simplicial complex $X$~\cite{wu2023simplicial} is a collection of unoriented $n$-dimensional objects, called simplices in $X$, represented by sets. Formally, let $X \subseteq 2^S$ for some finite set $S$. The set $X$ is called a simplicial complex if $\tau \in X$ and $\sigma \subset \tau$, then $\sigma \in X$. An element $\sigma \in X$ is referred to as $n$-simplex if $\Vert \sigma \Vert = n$. We denote $X_n$ as the set of $n$-simplices in $X$. The vertex set of the simplicial complex is defined by $V(X) \coloneqq {\bigcup}_{\sigma \in X} \ \sigma$. A common method for uniformly orienting the simplices of a simplicial complex is by imposing a total order on the vertex set $V(X)$.\\

\begin{definition}Let $X$ be a simplicial complex.
\begin{enumerate}
    \item $(X, <)$ is called an \textit{ordered simplicial complex} if $<$ is a total order on the vertex set $V(X)$.
    \item Suppose $(X,<)$ is an ordered simplicial complex.
    \begin{enumerate}
        \item The $i$th facet map, denoted by $d^n_i : X_n \to X_{n-1}$, is defined by
        \begin{equation}
        d^n_i(\{v_{j_0}, \cdots, v_{j_n} \}) \coloneqq \{v_{j_0}, \cdots, \widehat{v_{j_i}},\cdots, v_{j_n} \}    
        \end{equation}
        where the vertices satisfy $v_{j_0} < \cdots < v_{j_n}$. If $\sigma = d_i^n(\tau)$ for some $\tau$, we say $\sigma$ is a facet of $\tau$ (or equivalently, $\tau$ is a cofacet of $\sigma$), denoted by $\sigma \prec \tau$.
        \item The \textit{signed incidence} $[\sigma:\tau]$ between a simplex $\tau$ and its facet $\sigma=d_i^n(\tau)$~\cite{curry2014sheaves} is given by $[\sigma:\tau] \coloneqq (-1)^i$.
    \end{enumerate}
\end{enumerate}
\end{definition}
When $\sigma \prec \tau$, the signed incidence $[\sigma : \tau]$ indicates whether the orientation of the facet $\sigma$ is consistent, denoted by $+1$, or reversed, denoted by $-1$, relative to the orientation of $\tau$.\\

\begin{example}\label{ordscpxex}Let $V$ be a finite set and consider the simplicial complex $X = 2^{V} \setminus \emptyset$ with vertex set $V(X) = V$. An illustration of this simplicial complex $X$ for $V = \{v_0, v_1, v_2\}$ is provided in Fig.~\ref{ordscpxfigure}. Suppose that $<$ is a total order on $V(X)$ defined by $v_0 < v_1 < v_2$. Then we have the following facet maps and signed incidences:

\begin{itemize}
    \item $d^2_0(\{v_0, v_1, v_2 \}) = \{v_1, v_2 \}$ and \[[\{v_1, v_2 \} : \{v_0, v_1, v_2 \}]=+1.\]
    \item $d^2_1(\{v_0, v_1, v_2 \}) = \{v_0, v_2 \}$ and \[[\{v_0, v_2 \} : \{v_0, v_1, v_2 \}]=-1.\]
    \item $d^2_2(\{v_0, v_1, v_2 \}) = \{v_0, v_1 \}$ and \[[\{v_0, v_1 \} : \{v_0, v_1, v_2 \}]=+1.\]
\end{itemize}

\begin{figure}
    \centering
        \begin{tikzpicture}
        [
            he/.style={draw, ellipse, dotted, inner sep=10pt},        
            node/.style={circle, fill=black, inner sep=0pt, minimum size=1.2mm, label distance=2mm} 
        ]
        \node[node, label=above:{\footnotesize $v_0$}] (v0) at (-1.5,0.75) {};
        \node[node, label=below:{\footnotesize $v_1$}] (v1) at (0,0.1875) {};
        \node[node, label=above:{\footnotesize $v_2$}] (v2) at (0,1.5) {};
        \node[draw=none, fill=none] (g1) at (0, 0.6) {};
        \node[draw=none, fill=none] (g2) at (0, 0.8) {};
        \node[label=left:{\footnotesize $V$}] [he, fit=(v0) (v1) (v2)] {};
        \fill[gray!20] (v0.center) -- (v1.center) -- (v2.center) -- cycle;
        \draw (v0) to node[label=below:{\footnotesize $+1$}] {} (v1);
        \draw (v1) -- (g1);
        \draw (g2) to node[font=\footnotesize, pos=0.3, xshift=3mm, yshift=1mm] {$+1$} (v2);
        \draw (v2) to node[label=above:{\footnotesize $-1$}] {} (v0);
        \node at (-0.3,0.7) {\footnotesize $\{v_0, v_1, v_2\}$};
        \end{tikzpicture}
        \caption{Illustration of the simplicial complex $2^V \setminus \emptyset$.}
        \label{ordscpxfigure}
\end{figure}
\end{example}

\subsection{Symmetric simplicial sets}

A symmetric simplicial set $X$~\cite{grandis2001finite} is a collection of oriented $n$-dimensional objects, called simplices in $X$, represented by tuples, such that each simplex has well-defined facets. \\

\begin{definition}\label{symsSet}Let $X = \{X_n\}_{n \in \mathbb{N}}$ be a collection of finite sets.
\begin{enumerate}
    \item The collection $X$ is called a \textit{symmetric simplicial set} if it is equipped with a family of functions 
    \begin{equation}
    \{X(\mu) : X_n \to X_m \}_{\{\mu : [m] \to [n]\}}    
    \end{equation}
    satisfying if $\mu : [m] \to [n], \nu : [n] \to [p]$ then 
        \begin{equation}\label{identityforsymsSet}
          X(\nu \circ \mu) = X(\mu) \circ X(\nu).   
        \end{equation}
    An element of $X_n$ is called $n$-simplex in $X$. We say two $n$-simplices $\sigma, \sigma'$ are \textit{equivalent} if 
    \begin{equation}\label{equivalentsimplices}
    X(g)(\sigma')=\sigma
    \end{equation}
    for some $g \in \mathfrak{S}_n$. We denote the equivalence class of $\sigma$ by $[\sigma]$.
    \item For an $n$-simplex $\sigma$ in $X$, $i \in [n]$, 
        \begin{enumerate}
            \item the \textit{$i$th vertex} of $\sigma$, denoted by $\sigma_i$, is defined as 
                \begin{equation}
                    X\left((i)_{[n]} \right)(\sigma) \in X_0
                \end{equation}
                where $(i)_{[n]} : [0] \to [n]$ is a function with $(i)_{[n]}(0)=i$. 
                $(\sigma_0, \cdots, \sigma_n)$ is called the \textit{tuple representation} of $\sigma$.
            \item the \textit{$i$th facet} of $\sigma$, denoted by $d^n_i(\sigma)$, is defined as
                \begin{equation}
                    X\left(\delta^n_i \right)(\sigma) \in X_{n-1}
                \end{equation}
                where $\delta^n_i : [n-1] \to [n]$ is the unique order preserving injection from $\{0, \cdots, n-1\}$ to $\{0, \cdots, \hat{i}, \cdots, n \}$.
                If $\sigma = d_i^n(\tau)$ for some $\tau$, we say $\sigma$ is a facet of $\tau$ (or equivalently, $\tau$ is a cofacet of $\sigma$), denoted by $\sigma \prec \tau$.
            \item The \textit{signed incidence} $[\sigma:\tau]$ between a simplex $\tau$ and its facet $\sigma=d_i^n(\tau)$ is given by $[\sigma:\tau] \coloneqq (-1)^i$.
            \item Let $\varsigma^n_i : [n+1] \to [n]$ be the surjection defined by
            \begin{equation}
            \varsigma^n_i \coloneqq \begin{cases}
                    j, & \text{if $j \leq i$}\\
                    j-1, & \text{if $j > i$.}
                \end{cases}    
            \end{equation}
        $\sigma$ is said to be \textit{degenerate} if $\sigma$ lies in the image of 
        \begin{equation}
        X(\varsigma_i^{n-1}) : X_{n-1} \to X_{n}    
        \end{equation}
        for some $i \in [n-1]$. $\sigma$ is said to be \textit{nondegenerate} if $\sigma$ is not degenerate.
        \end{enumerate}
\end{enumerate}
\end{definition}

An $n$-simplex $\sigma$ has a tuple representation $(\sigma_0, \cdots, \sigma_n)$. Two $n$-simplices $\sigma, \sigma'$ are equivalent if $(\sigma_0, \cdots, \sigma_n)$ is equal to $(\sigma'_{g(0)}, \cdots, \sigma'_{g(n)})$ for some permutation $g \in \mathfrak{S}_n$. The equivalence class $[\sigma]$ is represented by a set $\{\sigma_0, \cdots, \sigma_n \}$. The facet $d_i^n(\sigma)$ is a $(n-1)$-simplex with tuple representation $(\sigma_0, \cdots, \widehat{\sigma_i}, \cdots, \sigma_n)$. The nondegeneracy of $\sigma$ implies $\sigma_i \neq \sigma_j$ for all $i \neq j$. These combinatorial structures of $\sigma$ are systematically encoded via maps $\mu : [m] \to [n]$, without explicitly referencing the tuple representation of $\sigma$.\\

\begin{example}\label{symsSetex}For a finite set $V$, the $V$-simplex $\Delta[V]$ is the symmetric simplicial set defined by the following data~\cite{choi2024cellular}:
\begin{itemize}
    \item $\Delta[V]_n$ is the set of all functions from $[n]$ to $V$. Hence an $n$-simplex of $\Delta[V]$ is a function 
    \begin{equation}
      (v_{j_0}, \cdots, v_{j_n})_V : [n] \to V  
    \end{equation}
    whose image of $i \in [n]$ is $v_{j_i}$.
    \item For any function $\mu : [m] \to [n]$, the induced map $\Delta[V](\mu) : \Delta[V]_n \to \Delta[V]_m$ is defined as 
    \begin{equation}\label{identityforVsimplex}
    \Delta[V](\mu)((v_{j_0}, \cdots, v_{j_n})_V) \coloneqq (v_{j_{\mu(0)}}, \cdots, v_{j_{\mu(m)}})_V.    
    \end{equation}
    Equation \eqref{identityforVsimplex} demonstrates that $\Delta[V]$ satisfies \eqref{identityforsymsSet}, confirming that $\Delta[V]$ is indeed a symmetric simplicial set.
\end{itemize}
An illustrative example of this symmetric simplicial set structure is provided in Fig.~\ref{symsSetfigure}.  Here are some characteristics of an $n$-simplex $(v_{j_0}, \cdots, v_{j_n})_V$:
\begin{itemize}
    \item Tuple representation of $(v_{j_0}, \cdots, v_{j_n})_V$ is $(v_{j_0}, \cdots, v_{j_n})$.
    \item The $i$th vertex of $(v_{j_0}, \cdots, v_{j_n})_V$ is $v_{j_i}$.
    \item The $i$th facet of $(v_{j_0}, \cdots, v_{j_n})_V$ is $(v_{j_0}, \cdots, \widehat{v_{j_i}} \cdots, v_{j_n})_V$.
    \item A simplex $(v_{j_0}, \cdots, v_{j_n})_V$ is degenerate if and only if $v_{j_i}=v_{j_{i'}}$ for some $i \neq i'$. In particular, the degenerate $n$-simplex $(v, \cdots, v)_V$ is called the $v$ \textit{of multiplicity} $(n+1)$.
\end{itemize}
Hence, an $n$-simplex in $\Delta[V]$ can be naturally interpreted as an ordered $(n+1)$-tuple of elements in $V$. \\
\end{example}

In Examples \ref{ordscpxex} and \ref{symsSetex}, we constructed an ordered simplicial complex and a symmetric simplicial set, respectively, from $V$. These two constructions differ as follows:

\begin{itemize}
    \item A simplex in $2^V \setminus \emptyset$ is represented as an unordered subset of $V$, whereas a simplex in $\Delta[V]$ is represented as an ordered tuple in $V$.
    
    \item $2^V \setminus \emptyset$ requires a total order on the vertex set $V$ to define the $i$th facet of a simplex explicitly. In contrast, $\Delta[V]$ does not depend on any choice of total order on $V$.
    
    \item In $2^V \setminus \emptyset$, no $n$-simplices exist for $n > \Vert V \Vert$. In contrast, $\Delta[V]$ contains $n$-simplices for every $n \in \mathbb{N}$, due to the allowance of degenerate simplices.
\end{itemize}

\begin{figure}
    \centering
        \begin{tikzpicture}
        [
            he/.style={draw, ellipse, dotted, inner sep=10pt},        
            node/.style={circle, fill=black, inner sep=0pt, minimum size=1.2mm, label distance=2mm} 
        ]
        \node[node, label=above:{\footnotesize $v_0$}] (v0) at (-1.5,0.75) {};
        \node[node, label=below:{\footnotesize $v_1$}] (v1) at (0,0.1875) {};
        \node[node, label=above:{\footnotesize $v_2$}] (v2) at (0,1.5) {};
        \node[draw=none, fill=none] (g1) at (0, 0.6) {};
        \node[draw=none, fill=none] (g2) at (0, 0.8) {};
        \node[label=left:{\footnotesize $V$}] [he, fit=(v0) (v1) (v2)] {};
        \fill[gray!20] (v0.center) -- (v1.center) -- (v2.center) -- cycle;
         \draw[-latex] (v0) to node[label=below:{\footnotesize $+1$}] {} (v1);
         \draw[-latex] (g2) to node[font=\footnotesize, pos=0.3, xshift=3mm, yshift=1mm] {$+1$} (v2);
         \draw[-latex] (v0) to node[label=above:{\footnotesize $-1$}] {} (v2);
        \draw (v1) -- (g1);
        \draw[->] (-0.7,0.7) arc[start angle=-135, end angle=150, radius=0.12];
        \node at (0.5,0.7) {\footnotesize $(v_0, v_1, v_2)_V$};
        \end{tikzpicture}
        \caption{Illustration of the symmetric simplicial set $\Delta[V]$.}
        \label{symsSetfigure}
\end{figure}

\subsection{Sheaf Laplacians}
In this subsection, we assume $X$ is either an ordered simplicial complex or a symmetric simplicial set. We define upper and lower adjacency relations between simplices of the same dimension in $X$ motivated by ~\cite{bodnar2021weisfeilercell}.\\

\begin{definition}
Let $\sigma, \sigma' \in X_n$ for some $n \in \mathbb{N}$.

\begin{enumerate}
    \item The simplices $\sigma$ and $\sigma'$ are said to be \textit{upper adjacent} if there exists $\tau \in X_{n+1}$ such that $\sigma \prec \tau$ and $\sigma' \prec \tau$.

    \item The simplices $\sigma$ and $\sigma'$ are said to be \textit{lower adjacent} if there exists $\mu \in X_{n-1}$ such that $\mu \prec \sigma$ and $\mu \prec \sigma'$.
\end{enumerate}
\end{definition}

In other words, two simplices are upper adjacent if they share a common cofacet, and they are lower adjacent if they share a common facet.\\

\begin{example}
Consider $\Delta[V]$ for $V = \{v_0, v_1, v_2\}$. Fig.~\ref{adjacencyforsymsSet} illustrates upper and lower adjacency relations between simplices:
\begin{itemize}
    \item The simplices $(v_0)_V$ and $(v_1)_V$ are upper adjacent since both are facets of the common simplex $(v_0, v_1)_V$.
    \item The simplices $(v_0, v_1)_V$ and $(v_1, v_2)_V$ are lower adjacent since both share the common facet $(v_1)_V$. \\
\end{itemize}
\begin{figure}
    \centering
        \begin{tikzpicture}
     [
            he/.style={draw, ellipse, dotted, inner sep=10pt},        
            node/.style={circle, fill=black, inner sep=0pt, minimum size=1.2mm, label distance=2mm} 
        ]
        \node[node, label=left:{\footnotesize $v_0$}] (v0) at (-1.5,0.75) {};
        \node[node, label=below:{\footnotesize $v_1$}] (v1) at (0,0.1875) {};
        \node[node, label=above:{\footnotesize $v_2$}] (v2) at (0,1.5) {};
        \node[draw=none, fill=none] (g1) at (0, 0.6) {};
        \node[draw=none, fill=none] (g2) at (0, 0.8) {};
        \node[label=left:{\footnotesize $V$}] [he, fit=(v0) (v1) (v2)] {};
        \draw[-latex] (v0) to node[midway, left, xshift = 0.5cm, yshift = -0.3cm] {\footnotesize $(v_0, v_1)_{V}$} (v1);
        \draw[-latex] (v1) to node[midway, right] {\footnotesize $(v_1, v_2)_{V}$} (v2);
        \end{tikzpicture}
    \caption{Illustration of adjacency relations in $\Delta[V]$.}
    \label{adjacencyforsymsSet}
\end{figure}
\end{example}

To utilize simplices for data analysis, it is essential to assign vector-valued features to each simplex. Moreover, given rules that transfer features from an $n$-simplex to its facets and vice versa, it becomes possible to diffuse features between adjacent simplices. Such feature diffusion mechanisms are formally captured by the concept of a cellular sheaf~\cite{choi2024cellular}.\\

\begin{definition}
A cellular sheaf $(X, \mathcal{F})$ of degree $m$ consists of the following data:
\begin{itemize}
    \item For $n \in [m]$, $n$-simplex $\sigma \in X_n$, a $\mathbb{R}$-vector space $\mathcal{F}(\sigma)$, called the \textit{stalk at $\sigma$}.
    \item For $n \in [m]$, $n$-simplex $\sigma \in X_n$ and facet $d_i^n(\sigma)$, a linear map
    \begin{equation}
        \mathcal{F}\left(d_i^n(\sigma) \prec \sigma\right) : \mathcal{F}\left(d_i^n(\sigma)\right) \rightarrow \mathcal{F}(\sigma)
    \end{equation}
    satisfying the following compatibility conditions:
    \begin{equation}\label{facemapidentity}
    \begin{split}
        &\mathcal{F}\left(d_j^{n-1}(d_i^n(\sigma)) \prec d_i^n(\sigma)\right) \circ \mathcal{F}\left(d_i^n(\sigma) \prec \sigma\right) \\
        & \qquad= \mathcal{F}\left(d_{i-1}^{n-1}(d_j^n(\sigma)) \prec d_j^n(\sigma)\right) \circ \mathcal{F}\left(d_j^n(\sigma) \prec \sigma\right)
    \end{split}
    \end{equation}
    for every $n \in [m]$, $i \in [n]$, $j<i$, and $\sigma \in X_n$.
\end{itemize}
Elements of $\mathcal{F}(\sigma)$ are called \textit{features at simplex $\sigma$}.\\
\end{definition}

A cellular sheaf $\mathcal{F}$ of degree $m$ on $X$ assigns a vector space to each $n$-simplex for $n \in [m]$, together with linear maps between these vector spaces corresponding to the facet relations. When $m$ is sufficiently large, we may omit specifying its value. Intuitively, the map $\mathcal{F}^*(d_i^n(\sigma) \prec \sigma)$ describes how features at the simplex $\sigma$ are restricted to its $i$th facet, while the map $\mathcal{F}(d_i^n(\sigma) \prec \sigma)$ describes how features from the $i$th facet are propagated upward to $\sigma$.

Using the cellular sheaf structure, a feature $x_{\sigma} \in \mathcal{F}(\sigma)$ at simplex $\sigma$ can be diffused to an adjacent simplex $\sigma'$ as follows. Suppose $\sigma$ and $\sigma'$ are upper adjacent with common cofacet $\tau$. Then, the feature $x_{\sigma}$ can be diffused to $\sigma'$ by
\begin{equation}
(-1)^{[\sigma : \tau] + [\sigma' : \tau]} \mathcal{F}^*(\sigma' \prec \tau)\,\mathcal{F}(\sigma \prec \tau)(x_{\sigma}).
\end{equation}

Similarly, if $\sigma$ and $\sigma'$ are lower adjacent with common facet $\mu$, the feature $x_{\sigma}$ can be diffused to $\sigma'$ by
\begin{equation}
(-1)^{[\mu : \sigma] + [\mu : \sigma']} \mathcal{F}(\mu \prec \sigma')\,\mathcal{F}^*(\mu \prec \sigma)(x_{\sigma}).
\end{equation}

These feature diffusions between adjacent simplices are illustrated in Fig.~\ref{diffusionoffeaturestoadjacentfaces}. Aggregating all such possible feature diffusions according to these rules results in the construction of the \textit{sheaf Laplacian}.\\

\begin{figure}
\centering
\begin{tikzcd}
& \mathcal{F}(\tau) \arrow[dr, "(-1)^{[\sigma': \tau]}\mathcal{F}^*(\sigma'\prec \tau)"] & \\
\mathcal{F}(\sigma) \arrow[ur, "(-1)^{[\sigma : \tau]}\mathcal{F}(\sigma\prec \tau)"] \arrow[dr, "(-1)^{[\mu : \sigma]}\mathcal{F}^*(\mu \prec \sigma)", swap] &   & \mathcal{F}(\sigma')\\
& \mathcal{F}(\mu) \arrow[ur, , "(-1)^{[\mu : \sigma']}\mathcal{F}(\mu \prec \sigma')", swap]&
\end{tikzcd}    
\caption{Illustration of feature diffusion between adjacent simplices $\sigma$ and $\sigma'$.}

\label{diffusionoffeaturestoadjacentfaces}
\end{figure}

\begin{definition}Let $\mathcal{F}$ be a cellular sheaf on $X$ and $k \in \mathbb{N}$.
    \begin{enumerate}
        \item A \textit{$k$-cochain}, denoted by $\mathbf{x} = (x_\sigma)_{\sigma \in X_k}$, is an element of the direct sum of stalks over all $k$-simplices in $X$. We call $x_\sigma$ the $\sigma$-component of the cochain $\mathbf{x}$. The set of all $k$-cochains is denoted by $C^k(X,\mathcal{F})$.
        \item For a $k$-cochain $\mathbf{x}$, the $\sigma$-component of the \textit{degree $k$ sheaf Laplacian} $L_{\mathcal{F}}^k(\mathbf{x})_{\sigma}$ is defined by
        \begin{equation}\label{uppersheaflaplacian}
        \begin{split}
            &\sum_{\sigma'', \tau}(-1)^{[\sigma : \tau]+[\sigma' : \tau]}\mathcal{F}^*(\sigma \prec \tau)\mathcal{F}(\sigma' \prec \tau)(x_{\sigma'}) \\
            & \ +\sum_{\sigma'', \mu}(-1)^{[\mu : \sigma]+[\mu : \sigma'']}\mathcal{F}(\mu \prec \sigma)\mathcal{F}^*(\mu \prec \sigma'')(x_{\sigma''})
        \end{split}
        \end{equation}
        where \begin{itemize}
                \item $\sigma'$ ranges over all simplices upper adjacent to $\sigma$
                \item $\sigma''$ ranges over all simplices lower adjacent to $\sigma$
                \item $\tau$ ranges over all common cofacets of $\sigma$ and $\sigma'$
                \item $\mu$ ranges over all common facets of $\sigma$ and $\sigma''$.
              \end{itemize}
 Additionally, the $\sigma$-component of the diagonal blocks, denoted by $D^k_{\mathcal{F}}(\mathbf{x})_{\sigma}$, is defined as
        \begin{equation}\label{upperdiagonalizedsheaflaplacian}
        \begin{split}
            & \sum_{\{\tau \mid \sigma \prec \tau \}}\mathcal{F}^*(\sigma \prec \tau)\mathcal{F}(\sigma \prec \tau)(x_{\sigma}) \\
            & \quad + \sum_{\{\mu \mid \mu \prec \sigma \}}\mathcal{F}(\mu \prec \sigma)\mathcal{F}^*(\mu \prec \sigma)(x_{\sigma}).
        \end{split}
        \end{equation}
        \item For a $k$-cochain $\mathbf{x}$, the $\sigma$-component of the \textit{normalized degree $k$ sheaf Laplacian}, denoted by $\mathcal{L}^k_{\mathcal{F}}(\mathbf{x})_{\sigma}$, is defined by
        \begin{equation}\label{normalizedsheafLaplacian}
            \left((D^k_{\mathcal{F}})^{-\frac{1}{2}}L^k_{\mathcal{F}}(D^k_{\mathcal{F}})^{-\frac{1}{2}}\right)(\mathbf{x})_{\sigma}.
        \end{equation}
    \end{enumerate}

If $X$ is an ordered simplicial complex with a specified total order $<$, we explicitly denote the sheaf Laplacian as $L_{\mathcal{F},<}^k$.\\
\end{definition}

Given a cellular sheaf $\mathcal{F}$, a $k$-cochain provides an assignment of features to every $k$-simplex in $X$. The degree $k$ sheaf Laplacian characterizes the diffusion of features from each $k$-simplex to its adjacent simplices through $\mathcal{F}$.\\

\begin{example}\label{SheafLaplacianVsimplex}Let $X = \Delta[V]$ be a $V$-simplex and $\mathcal{F}$ be a cellular sheaf on $X$. Suppose $\mathbf{x}$ is a 0-cochain and $v, w \in V$. Then the $(v)_V$-component of the degree 0 sheaf Laplacian can be computed as follows:
\begin{itemize}
    \item The simplices $(v)_V$ and $(w)_V$ are upper adjacent with common cofacets $(v,w)_V$ and $(w,v)_V$, satisfying the signed incidences:
    \begin{enumerate}
        \item $[(v)_V : (v,w)_V] = [(w)_V : (w,v)_V]=-1$
        \item $[(w)_V : (v,w)_V] = [(v)_V : (w,v)_V]=+1$.
    \end{enumerate}
    \item The simplex $(v)_V$ is also upper adjacent to itself through the common cofacets $(v,w)_V$ and $(w,v)_V$, satisfying the signed incidences:
    \begin{enumerate}
        \item $[(v)_V : (v,w)_V]=-1$
        \item $[(v)_V : (w,v)_V]=+1$.
    \end{enumerate}
\end{itemize}
Hence, by \eqref{uppersheaflaplacian}, the $(v)_V$-component of the degree 0 sheaf Laplacian is expressed as:
\begin{equation}
\begin{split}
    &\sum_{w \in V}\mathcal{F}^*\left((v)_V \prec (v,w)_V\right)\mathcal{F}\left((v)_V \prec (v,w)_V\right)(x_{v})\\
    & +\sum_{w \in V}\mathcal{F}^*\left((v)_V \prec (w,v)_V\right)\mathcal{F}\left((v)_V \prec (w,v)_V\right)(x_{v})\\
    & -\sum_{w \in V}\mathcal{F}^*\left((v)_V \prec (v,w)_V\right)\mathcal{F}\left((w)_V \prec (v,w)_V\right)(x_{w})\\
    & - \sum_{w \in V}\mathcal{F}^*\left((v)_V \prec (w,v)_V\right)\mathcal{F}\left((w)_V \prec (w,v)_V\right)(x_{w}).
\end{split}
\end{equation}
\end{example}
\section{Methodology}
\label{sec:methodology}
In this section, we first demonstrate that adjacency between hyperedges in general hypergraphs is not inherently well-defined. We discuss the most natural approach to resolving this issue—using the induced simplicial complex—which provides a clear notion of adjacency. However, we will identify two significant problems arising from this approach.

\subsection{Problem definition}
In a graph, two nodes are considered adjacent if an edge directly connects them, and two edges are adjacent if they share a common node. The degree 0 sheaf Laplacian effectively describes how node features diffuse through edges, while the degree 1 sheaf Laplacian captures the diffusion of edge features across adjacent edges. These concepts accurately reflect feature diffusion dynamics in graphs, provided there is an adequate number of edges.

However, such conditions do not directly translate to hypergraphs. In hypergraphs, adjacency between two nodes requires the existence of a dimension 1 hyperedge explicitly connecting them. Consequently, in the absence of dimension 1 hyperedges, the degree 0 sheaf Laplacian becomes zero. Similarly, two hyperedges of dimension $n$ are adjacent if there exists a hyperedge of dimension $n+1$ containing both or if one hyperedge of dimension $n-1$ is contained within the other. Crucially, having many hyperedges in total does not necessarily imply the presence of numerous adjacent hyperedges. The following example clearly illustrates this limitation.\\

\begin{example}\label{exampleofhypergraphwithsheaflaplacianzero}
Consider a hypergraph $H = (V(H), E(H), f_H)$ defined by:
\begin{itemize}
    \item $V(H) \coloneqq \{v_0, v_1, v_2, v_3\}$, $E(H) \coloneqq \{e_0, e_1, e_2\}$
    \item $f_H(e_0) \coloneqq \{v_0, v_1, v_2 \}$
    \item $f_H(e_1) \coloneqq \{v_0, v_2, v_3 \}$
    \item $f_H(e_2) \coloneqq \{v_0, v_1, v_3 \}$.
\end{itemize}
This hypergraph is illustrated in Fig.~\ref{failureofadjacent}. All hyperedges have dimension 2, and there are no dimension 1 hyperedges connecting any two nodes. Thus, the degree 0 sheaf Laplacian is zero. Furthermore, since no hyperedges contain others or share common subsets, no hyperedges are adjacent, leading the degree $n$ sheaf Laplacian to vanish for all $n$.\\
\end{example}

\begin{figure}
    \centering
    \begin{tikzpicture}
        [
            he/.style={draw, ellipse, dotted, inner sep=10pt},        
            node/.style={circle, fill=black, inner sep=0pt, minimum size=1.2mm, label distance=2mm}, 
            tight fit/.style={inner xsep=5pt, inner ysep=5pt}       
        ]
        \node[draw=none, fill=none] (v0) at (-1.5,0.75) {};
        \node[node, label=below:{\footnotesize $v_1$}] (v1) at (-1,-0.1875) {};
        \node[node, label=below:{\footnotesize $v_0$}] (v2) at (0,0.1875) {};
        \node[node, label=above:{\footnotesize $v_2$}] (v3) at (0,1.5) {};
        \node[node, label=below:{\footnotesize $v_3$}] (v4) at (1,-0.1875) {};
        \node[draw=none, fill=none] (g1) at (1.5,0.75) {}; 
        \node[draw=none, fill=none] (g3) at (0,-1.5) {};

        \node[label=left:{\footnotesize $e_0$}] [he, fit=(v0) (v1) (v2) (v3)] {};
        \node[label=right:{\footnotesize $e_1$}] [he, fit=(v2) (v3) (v4) (g1)] {};
        \node[label=below:{\footnotesize $e_2$}] [he, tight fit, fit=(v2) (v1) (v4) (g3)] {}; 
    \end{tikzpicture}
    \caption{An example of a hypergraph that has no adjacent nodes or adjacent hyperedges.}
    \label{failureofadjacent}
\end{figure}

To address the above-mentioned problem with hypergraphs, one natural approach involves systematically inserting all subsets of intersections among hyperedges into the set of hyperedges. This construction yields an induced simplicial complex, which we denote by
\begin{equation}
S(H) \coloneqq \{S \mid S \subseteq f_H(e) \text{ for some } e \in E(H)\} \cup V(H).    
\end{equation}
The resulting complex $S(H)$ thus incorporates all possible unoriented sub-relations induced from the original hyperedge structure of $H$. The vertex set $V(S(H))$ coincides with $V(H)$. By imposing a total order on $V(H)$, the sheaf Laplacian on this complex becomes well-defined, effectively resolving the adjacency issues described earlier. However, adopting this solution introduces two new significant problems, which we outline next.\\
\begin{enumerate}
\item First, the original hypergraph $H$ cannot necessarily be recovered from the induced simplicial complex $S(H)$. In other words, non-isomorphic hypergraphs may induce identical simplicial complexes. For example, consider two hypergraphs $H = (\{v_0, v_1, v_2\}, \{e\}, f_H)$,  $H' = (\{v_0, v_1, v_2\}, \{e, e'\}, f_{H'})$ with labeling functions defined by
\begin{equation}
f_H(e) = f_{H'}(e) = \{v_0, v_1, v_2\},\quad f_{H'}(e') = \{v_0, v_1\}.    
\end{equation}
Clearly, the hypergraphs \(H\) and \(H'\) are non-isomorphic. Nevertheless, their induced simplicial complexes coincide:
\begin{equation}
S(H) = S(H') = 2^{\{v_0, v_1, v_2\}} \setminus \emptyset.    
\end{equation}
This ambiguity arises because it becomes impossible to determine whether $\{v_0, v_1\}$ originates as a subrelation of hyperedge $e$, derives from hyperedge $e'$, or exists as an independent hyperedge itself.

\item Second, the resulting sheaf Laplacian depends explicitly on the choice of the total order assigned to the vertex set~\cite{bodnar2021weisfeilertop}. To illustrate, consider a hypergraph \(H = (\{v_0, v_1, v_2\}, \{e\}, f_H)\) with a single hyperedge defined as \(f_H(e)= \{v_0, v_1, v_2\}\). Then its induced simplicial complex is
\begin{equation}
S(H) = 2^{\{v_0, v_1, v_2\}} \setminus \emptyset.
\end{equation}

Let $<_1$ and $<_2$ be two distinct total orders on the vertex set $V(S(H)) = \{v_0, v_1, v_2\}$, defined by $v_0 <_1 v_1 <_1 v_2$ and $v_1 <_2 v_0 <_2 v_2$. Suppose $\mathcal{F}$ is a cellular sheaf on $S(H)$ satisfying
\begin{equation}
\mathcal{F}(\{v_1\} \prec \{v_1, v_2\}) = \mathcal{F}(\{v_2\} \prec \{v_1, v_2\}) = 0.
\end{equation}
For simplicity, denote 
\begin{equation}
A(\{v_i, v_j\}) \coloneqq \mathcal{F}(\{v_i, v_j\} \prec \{v_0, v_1, v_2\})
\end{equation}
for $i, j \in [2]$. Then, for a given $\mathbf{x} \in \mathcal{F}(\{v_1, v_2\})$, we have
\begin{equation}
L^1_{\mathcal{F},<_1}(\mathbf{x})_{\{v_0, v_1\}} = A(\{v_0, v_1\})^* A(\{v_1, v_2\})(\mathbf{x}),
\end{equation}
whereas under the alternative order,
\begin{align*}
L^1_{\mathcal{F},<_2}(\mathbf{x})_{\{v_0, v_1 \}} &= -A(\{v_0, v_1 \})^*A(\{v_1, v_2 \})(\mathbf{x})\\
&\neq L^1_{\mathcal{F},<_1}(\mathbf{x})_{\{v_0, v_1 \}}.
\end{align*}

Therefore, $L^1_{\mathcal{F}, <_1} \neq L^1_{\mathcal{F}, <_2}$. \\
\end{enumerate}

The sheaf Laplacian is well-defined on symmetric simplicial sets without requiring a total order. Thus, if we can construct a symmetric simplicial set from a given hypergraph and demonstrate that this construction preserves the hypergraph's structural information, we can effectively resolve the aforementioned problems inherent to hypergraphs.

\subsection{Symmetric simplicial lifting}
Given a hyperedge $e$, the $f_H(e)$-simplex $\Delta[f_H(e)]$ is the collection of all possible oriented relations within $e$, represented by tuples, that can be predicted from the observed relation $e$. In general hypergraphs with multiple hyperedges, the $f_H(e)$-simplices from different hyperedge $e$ must be glued together along shared nodes. Crucially, since the original hypergraph $H$ cannot be reconstructed from $S(H)$ due to the lack of information about which hyperedge each possible relation originated from, the elements of $\Delta[f_H(e)]$ and $\Delta[f_H(e')]$ for different hyperedges $e, e'$ should not be identified.

For example, in Fig.~\ref{motivationofsymssetfromhypergraph}, $(v_2, v_3)_{f_H(e)} \in \Delta[f_H(e)]_1$ and $(v_2, v_3)_{f_H(e')} \in \Delta[f_H(e')]_1$ share the same image $\{v_2, v_3\}$ but must be treated as distinct objects. This distinction indicates $(v_2, v_3)_{f_H(e)}$ originates from $e$, while $(v_2, v_3)_{f_H(e')}$ originates from $e'$, hence can avoid the problem occurred in $S(H)$.

However, not all ordered relations are treated as distinct. For instance, $(v_2, v_2)_{f_H(e)}$, $(v_2, v_2)_{f_H(e')}$, and $(v_2, v_2)_{\{v_2\}}$ all represent the shared node $v_2$ as 2-tuples. Since we glue $\Delta[f_H(e)], \Delta[f_H(e')]$ along the shared nodes $\{v_2, v_3 \}$, these elements should be identified. Extending this principle to all hyperedges, nodes, and $n$-dimensional cases yields symmetric simplicial lifting $\Delta(H)$. \\

\begin{figure}
    \centering
 \begin{tikzpicture}
      [
            he/.style={draw, ellipse, dotted, inner sep=10pt},        
            node/.style={circle, fill=black, inner sep=0pt, minimum size=1.2mm, label distance=2mm} 
        ]
        \node[node, label=left:{\footnotesize $v_0$}] (w0) at (4.5,0.75) {};
        \node[node, label=below:{\footnotesize $v_1$}] (w1) at (4.5,-0.1875) {};
        \node[node, label=below:{\footnotesize $v_2$}] (w2) at (6,0.1875) {};
        \node[node, label=above:{\footnotesize $v_3$}] (w3) at (6,1.5) {};
        \node[node, label=right:{\footnotesize $v_4$}] (w4) at (7.5,0.9375) {};
        \node[draw=none, fill=none] (a1) at (6.9375,-0.1875) {}; 
        \node[draw=none, fill=none] (a2) at (4.875,-0.09375) {};
        \node[draw=none, fill=none] (a3) at (5.625,0.09375) {};
        \fill[gray!20] (w0.center) -- (w1.center) -- (w2.center) to[bend left] (w3.center) -- cycle;
        \fill[gray!20] (w2.center) to[bend right] (w3.center) -- (w4.center) -- cycle;
        \node[label=left:{\footnotesize $e$}][he, fit=(w0) (w1) (w2) (w3)] {};
        \node[label=right:{\footnotesize $e'$}][he, fit=(w2) (w3) (w4) (a1)] {};
        \node[xshift=4mm, yshift=-5mm, label={[label distance=0.5cm]above left:$(v_0, v_1, v_2, v_3)_{f_H(e)}$}, inner sep=0pt] at ($(w0)!0.5!(w3)$) {};
        \node[xshift=-8mm, yshift=-10mm, label={[label distance=0.5cm]above right:$(v_2, v_3, v_4)_{f_H(e')}$}, inner sep=0pt] at ($(w2)!0.6!(w4)$) {};
        \draw[-latex] (w0) -- (w1);
        \draw[-latex] (w1) -- (w2);
        \draw[-latex] (w0) -- (w3);
        \draw[-latex, dotted] (w0) -- (w2);
        \draw (w1) -- (a2) -- cycle;
        \draw (w1) -- (w3);
        \draw[-latex] (w3) -- (w4);
        \draw[-latex] (w2) -- (w4);
        \draw[-latex, bend left] (w2) to node[midway, left, xshift=0.1cm] {} (w3);
        \draw[-latex, bend right] (w2) to node[midway, right, xshift=-0.1cm] {} (w3);
        \end{tikzpicture}
    \caption{Motivation for symmetric simplicial lifting.}
    \label{motivationofsymssetfromhypergraph}
\end{figure}

\begin{definition}
Given a hypergraph $H$, we define its \textit{symmetric simplicial lifting} $\Delta(H)$ by the following data:
\begin{enumerate}
    \item $\Delta(H)_n$ is given by
    \begin{equation}
\left(\bigsqcup_{e \in E(H)} \Delta[f_H(e)]_n \right) \bigsqcup \left(\bigsqcup_{v \in V(H)} \Delta[\{v\}]_n \right)\Big/{\sim}
\end{equation}
where $\sim$ is the equivalence relation generated by 
\begin{equation*}
\Delta(H)_n \ni (v, \cdots, v)_{f_H(e)} \sim (v, \cdots, v)_{\{v\}} \in \Delta(H)_n 
\end{equation*}
for any $v \in e$. 
\begin{itemize}
    \item For a hyperedge $e \in E(H)$, the equivalence class of $(v_0, \cdots, v_n)_{f_H(e)}$ is denoted by $[v_0, \cdots, v_n]_e$.
    \item For a node $v \in V(H)$, the equivalence class of $(v, \cdots, v)_v$ is denoted by $[v, \cdots, v]_v$.
\end{itemize}
\item Given $\mu : [m] \to [n]$ and $\sigma \in \Delta(H)_n$, $\Delta(H)(\mu)(\sigma)$ is defined by
\begin{equation}\label{identityforsymsSetfromhypergraph}
\begin{cases}
    [v_{i_{\mu(0)}}, \cdots, v_{i_{\mu(m)}}]_e \in \Delta(H)_m, & \text{if } \sigma = [v_{i_0}, \cdots, v_{i_n}]_e \\
    [v, \cdots, v]_v \in \Delta(H)_m, & \text{if } \sigma = [v, \cdots, v]_v.
\end{cases}
\end{equation}

\end{enumerate}

Equation \eqref{identityforsymsSetfromhypergraph} demonstrates that $\Delta(H)$ satisfies \eqref{identityforsymsSet}, confirming that $\Delta(H)$ is indeed a symmetric simplicial set. We say that the simplex \([v_{i_0},\dots,v_{i_{\|e\|}}]_e \in \Delta(H)_{\Vert e \Vert}\) is a \textit{maximal nondegenerate \(\|e\|\)-simplex} if all vertices \(v_{i_l}\) are distinct for \(l \neq l'\).\\
\end{definition}

$\Delta[f_H(e)]$ is embedded into $\Delta(H)$ for each hyperedge $e$. For example, $m$-simplex $(v_{j_0}, \ldots, v_{j_m})_{f_H(e)}$ in $\Delta[f_H(e)]_m$ is embedded into $m$-simplex $[v_{j_0}, \ldots, v_{j_m}]_e$ in $\Delta(H)$. In particular, when $v \in f_H(e)$, the $v$ of multiplicity $(n+1)$, $(v, \cdots, v)_V$, is embedded into $[v, \ldots, v]_e$. We identify all $v$ of multiplicity $(n+1)$ from different hyperedges as $[v, \cdots, v]_v$. \\

\begin{example}\label{exsymssetfromhypergraph}Consider a hypergraph $H=(V(H), E(H), f_H)$, where the node set is $V(H) = \{v_0, v_1, v_2, v_3, v_4 \}$, the hyperedge set is $E(H) = \{e, e' \}$, and the labeling function is defined by $f_H(e) = \{v_0, v_1, v_2, v_3 \}$ and $f_H(e') = \{v_2, v_3, v_4 \}$. Fig.~\ref{symssetfromhypgph} illustrates several simplices of symmetric simplicial lifting $\Delta(H)$, including all 0-simplices, some 1-simplices (such as $[v_2, v_3]_e$, $[v_2, v_3]_{e'}$, and $[v_4, v_3]_{e'}$), and a 2-simplex, $[v_1, v_2, v_0]_e$.\\ 
\end{example}

Given the symmetric simplicial lifting $\Delta(H)$ in Example~\ref{exsymssetfromhypergraph}, we can construct a hypergraph $(V, E, f)$ (Fig.~\ref{hypergraphinducedbysymsset}) using solely the data from $\Delta(H)$, through the following systematic procedure:

\begin{itemize}
    \item Each $0$-simplex in $\Delta(H)$ becomes a node, hence
    \begin{equation}
    V \coloneqq \Delta(H)_0.
    \end{equation}
    \item Each equivalence class of maximal nondegenerate simplex becomes a hyperedge, hence
    \begin{equation}
    E \coloneqq \{ [[v_0, v_1, v_2, v_3]_e],\ [[v_2, v_3, v_4]_{e'}] \}.    
    \end{equation}
    \item Labeling function $f$ sends each equivalence class to its underlying set.
    \begin{enumerate}
        \item $[[v_0, v_1, v_2, v_3]_e] \mapsto \{v_0, v_1, v_2, v_3\}$
        \item $[[v_2, v_3, v_4]_{e'}] \mapsto \{v_2, v_3, v_4\}$. \\
    \end{enumerate}
\end{itemize}

We can easily demonstrate the isomorphism between hypergraphs $H$ and $(V,E,f)$. This reconstruction method extends naturally to arbitrary hypergraphs, thereby establishing a general framework for recovering the hypergraph structure from symmetric simplicial set data.\\

\begin{figure}
    \centering
    \begin{subfigure}[b]{0.45\textwidth}
        \centering
        \begin{tikzpicture}
      [
            he/.style={draw, ellipse, dotted, inner sep=10pt},        
            node/.style={circle, fill=black, inner sep=0pt, minimum size=1.2mm, label distance=2mm} 
        ]
        \node[node, label=left:{\footnotesize $[v_0]_e$}] (w0) at (4.5,0.75) {};
        \node[node, label=below:{\footnotesize $[v_1]_e$}] (w1) at (4.5,-0.1875) {};
        \node[node, label=right:{\footnotesize $[v_2]_e$}] (w2) at (6,0.1875) {};
        \node[node, label=above:{\footnotesize $[v_3]_e$}] (w3) at (6,1.5) {};
        \node[node, label=right:{\footnotesize $[v_4]_{e'}$}] (w4) at (7.5,0.9375) {};
        \node[draw=none, fill=none] (a1) at (6.9375,-0.1875) {}; 
        \node[draw=none, fill=none] (a2) at (4.875,-0.09375) {};
        \node[draw=none, fill=none] (a3) at (5.625,0.09375) {};

        \node[he, fit=(w0) (w1) (w2) (w3)] {};
        \node[he, fit=(w2) (w3) (w4) (a1)] {};
        \draw[-latex] (w4) to node[midway, right, xshift = -0.1cm, yshift = 0.2cm] {\footnotesize $[v_4, v_3]_{e'}$} (w3);
        \draw[-latex, bend left] (w2) to node[midway, left, xshift=0.1cm] {\footnotesize $[v_2,v_3]_e$} (w3);
        \draw[-latex, bend right] (w2) to node[midway, right, xshift=-0.1cm] {\footnotesize $[v_2,v_3]_{e'}$} (w3);
        \draw[-latex] (w0) -- (w1);
        \draw (w1) -- (a2);
        \draw[-latex] (a3) -- (w2);
        \draw[-latex] (w2) -- (w0);
        \draw[-latex] (5,0.25) arc[start angle=-135, end angle=150, radius=0.12];
        \node at (5.25,0) {\footnotesize $[v_1, v_2, v_0]_e$};
        \end{tikzpicture}
        \caption{$\Delta(H)$}
        \label{symssetfromhypgph}
    \end{subfigure}
    \hfill
    \begin{subfigure}[b]{0.45\textwidth}
        \centering
        \begin{tikzpicture}
   [
            he/.style={draw, ellipse, dotted, inner sep=10pt},        
            node/.style={circle, fill=black, inner sep=0pt, minimum size=1.2mm, label distance=2mm} 
        ]
        \node[node, label=left:{\footnotesize $v_0$}] (w0) at (4.5,0.75) {};
        \node[node, label=below:{\footnotesize $v_1$}] (w1) at (4.5,-0.1875) {};
        \node[node, label=below:{\footnotesize $v_2$}] (w2) at (6,0.1875) {};
        \node[node, label=above:{\footnotesize $v_3$}] (w3) at (6,1.5) {};
        \node[node, label=right:{\footnotesize $v_4$}] (w4) at (7.5,0.9375) {};
        \node[draw=none, fill=none] (a1) at (6.9375,-0.1875) {}; 
        \node[draw=none, fill=none] (a2) at (4.875,-0.09375) {};
        \node[draw=none, fill=none] (a3) at (5.625,0.09375) {};

    \fill[gray!15] (w0.center) -- (w1.center) -- (w2.center) to[bend left] (w3.center) -- cycle;

    \fill[gray!25] (w2.center) to[bend right] (w3.center) -- (w4.center) -- cycle;

        \node[he, fit=(w0) (w1) (w2) (w3)] {};
        \node[he, fit=(w2) (w3) (w4) (a1)] {};

        \draw[-latex] (w0) -- (w1);
        \draw[-latex] (w1) -- (w2);
        \draw[-latex] (w0) -- (w3);
        \draw[-latex, dotted] (w0) -- (w2);
        \draw (w1) -- (a2) -- cycle;
        \draw (w1) -- (w3);
        \draw[-latex] (w3) -- (w4);
        \draw[-latex] (w2) -- (w4);
        \draw[-latex, bend left] (w2) to node[midway, left, xshift=0.1cm] {} (w3);
        \draw[-latex, bend right] (w2) to node[midway, right, xshift=-0.1cm] {} (w3);
        \node[xshift=4mm, yshift=-5mm, label={[label distance=0.5cm]above left:$[[v_0, v_1, v_2, v_3]_e]$}, inner sep=0pt] at ($(w0)!0.5!(w3)$) {};
        \node[xshift=-8mm, yshift=-10mm, label={[label distance=0.5cm]above right:$[[v_2, v_3, v_4]_{e'}]$}, inner sep=0pt] at ($(w2)!0.6!(w4)$) {};
        \end{tikzpicture}
        \caption{$H(\Delta(H))$}
        \label{hypergraphinducedbysymsset}
    \end{subfigure}
    \caption{(a) Symmetric simplicial lifting $\Delta(H)$. (b) Recovered hypergraph $H(\Delta(H))$ is isomorphic to $H$.}
    \label{figurehypergraph}
\end{figure}

\begin{definition}Let $H$ be a hypergraph. The recovered hypergraph $H(\Delta(H))=(V,E,f)$ is a hypergraph with the following data:
    \begin{enumerate}
        \item $V$ is $V(H)$, the set of nodes of $H$.
        \item $E$ is the set of equivalence classes of maximal nondegenerate $n$-simplices for $n > 0$.
        \item For $[[v_{i_0}, \cdots, v_{i_{\Vert e \Vert}}]_e] \in E$, the labeling function $f$ is defined by 
        \begin{equation}
        f([[v_{i_0}, \cdots, v_{i_{\Vert e \Vert}}]_e]) \coloneqq \{v_{i_0}, \cdots, v_{i_{\Vert e \Vert}} \}.    
        \end{equation}
    \end{enumerate}
\end{definition}

\begin{proposition} For any hypergraph $H$, there exists a hypergraph isomorphism $H(\Delta(H)) \cong H$ that sends each equivalence class of a maximal nondegenerate simplex to its underlying set.
\end{proposition}
\begin{proof}
Define a map $b : E \to E(H)$ by \[b([[v_{i_0},\cdots, v_{i_{\Vert e \Vert}}]_e]) \coloneqq e.\]
Let $e' \in E(H)$ with $f_H(e') = \{v_0, \cdots, v_{\Vert e' \Vert} \}$.
\begin{itemize}
    \item Surjectivity: $[[v_0, \cdots, v_{\Vert e' \Vert}]_{e'}] \in b^{-1}(e')$ implies $b$ is surjective.
    \item Injectivity: Suppose 
    \begin{equation}
    b([[v_{i_0}, \cdots, v_{i_{\Vert e' \Vert}}]_{e'}]) = b([[v_{j_0}, \cdots, v_{j_{\Vert e' \Vert}}]_{e'}])=e'.
    \end{equation}
    By definition of $b$, this implies 
    \begin{equation}
    \{v_{i_0}, \cdots, v_{i_{\Vert e' \Vert}} \} = \{v_{j_0}, \cdots, v_{j_{\Vert e' \Vert}} \}    
    \end{equation}
    and thus 
    \begin{equation}
    [[v_{i_0}, \cdots, v_{i_{\Vert e' \Vert}}]_{e'}] = [[v_{j_0}, \cdots, v_{j_{\Vert e' \Vert}}]_{e'}].    
    \end{equation}
    Hence $b$ is injective.
\end{itemize}
Direct computations show that
\begin{align*}
((\operatorname{Id})_* \circ f)([[v_{0},\cdots, v_{{\Vert e' \Vert}}]_{e'}])&= (f_H \circ b)([[v_{0},\cdots, v_{{\Vert e' \Vert}}]_{e'}])\\
&= \{v_{0}, \cdots, v_{{\Vert e' \Vert}}\}
\end{align*}
and
\begin{align*}
((\operatorname{Id}^{-1})_* \circ f_H)(e') &= (f \circ b^{-1})(e')\\
&=\{v_0, \cdots, v_{\Vert e' \Vert}\}.
\end{align*}

Therefore, $(\operatorname{Id}, b) : H(\Delta(H)) \to H$ is a hypergraph isomorphism.\hfill $\square$ \\

\end{proof}

This connection is particularly valuable, as $\Delta(H)$ possesses a hierarchical structure compared to the original hypergraph $H$. The enhanced complexity of $\Delta(H)$ allows for a more nuanced analysis of the underlying hypergraph structure without losing the original hypergraph information $H$.

In real-world hypergraphs, node features are often the only provided data. To leverage these features effectively, it is critical to explicitly compute the degree 0 sheaf Laplacian, which models how node features diffuse to adjacent nodes through the hypergraph structure. \\

\begin{proposition}Suppose $H=(V(H), E(H), f_H)$ is a hypergraph and $\mathcal{F}$ is a cellular sheaf on $\Delta(H)$. For a 0-cochain $\mathbf{x}$, the $[v]_v$-component of the degree 0 sheaf Laplacian is

\begin{equation}
\begin{split}
    &\sum_{\substack{\{w \mid w,v \in e \ \textrm{for} \ e\} \\ \{e \mid v,w \in e \}}}\mathcal{F}^*([v]_v \prec [v,w]_e)\mathcal{F}([v]_v \prec [v,w]_e)(x_{v})\label{Laplacianformula}\\
    & +\sum_{\substack{\{w \mid w,v \in e \ \textrm{for} \ e\} \\ \{e \mid v,w \in e \}}}\mathcal{F}^*([v]_v \prec [w,v]_e)\mathcal{F}([v]_v \prec [w,v]_e)(x_{v})\\
    & -\sum_{\substack{\{w \mid w,v \in e \ \textrm{for} \ e\} \\ \{e \mid v,w \in e \}}}\mathcal{F}^*([v]_v \prec [v,w]_e)\mathcal{F}([w]_w \prec [v,w]_e)(x_{w})\\
    & -\sum_{\substack{\{w \mid w,v \in e \ \textrm{for} \ e\} \\ \{e \mid v,w \in e \}}}\mathcal{F}^*([v]_v \prec [w,v]_e)\mathcal{F}([w]_w \prec [w,v]_e)(x_{w}).
\end{split}
\end{equation}
\end{proposition}

\begin{proof}
\begin{itemize}
    \item $[v]_v, [w]_w$ are upper adjacent whose cofacets are $[v,w]_e, [w,v]_e$ for any $e \in E(H)$ such that $v,w \in f_H(e)$. Moreover,
    \begin{enumerate}
        \item $[[v]_v : [v,w]_e] = [[w]_w : [w,v]_e]=-1$
        \item $[[w]_w : [v,w]_e] = [[v]_v : [w,v]_e]=+1$.
    \end{enumerate}
    \item $[v]_v, [v]_v$ are upper adjacent whose cofacets are $[v,w]_e, [w,v]_e$ for any $e \in E(H)$ such that $v,w \in f_H(e)$. Moreover,
    \begin{enumerate}
        \item $[(v)_V : (v,w)_V]=-1$
        \item $[(v)_V : (w,v)_V]=+1$.
    \end{enumerate}
\end{itemize}
The degree 0 sheaf Laplacian easily follows by \eqref{uppersheaflaplacian}.\hfill $\square$ \\
\end{proof}

Given a graph $G$, we have two sheaf Laplacians: \eqref{sheafLaplacianongraph} and \eqref{Laplacianformula}. The next proposition says the two are closely related. \\

\begin{proposition}\label{sheaflaplacianongraphandassociatedsymsset}Suppose $G$ is a graph and $\mathcal{F}$ is a cellular sheaf on $G$. Then $\mathcal{F}$ induces a cellular sheaf $\widehat{\mathcal{F}}$ on $\Delta(G)$ satisfying 
\begin{equation}
\mathcal{L}^0_{\widehat{\mathcal{F}}} = \mathcal{L}_{\mathcal{F}}.    
\end{equation}
\end{proposition}

\begin{proof}
Define a cellular sheaf $\widehat{\mathcal{F}}$ on $\Delta(G)$ as follows. First, stalks are defined by
\begin{equation}
\widehat{\mathcal{F}}([v_{i_0}, \cdots, v_{i_n}]_e) \coloneqq \begin{cases}
\mathcal{F}(v) & \text{if $\{v_{i_0}, \cdots, v_{i_n} \}=\{v\}$ } \\
\mathcal{F}(e) & \text{otherwise.} 
\end{cases}
\end{equation}

Linear map $\widehat{\mathcal{F}}([v_{i_0}, \cdots,\widehat{v_{i_l}} ,\cdots, v_{i_n}]_e \prec [v_{i_0}, \cdots, v_{i_n}]_e)$ is defined by
\begin{equation}
\begin{cases}
  \mathcal{F}(v \in e) & \text{if $\{v_{i_0}, \cdots,\widehat{v_{i_l}} ,\cdots, v_{i_n}\} \neq \{v_{i_0}, \cdots, v_{i_n} \}$ }
  \\
  \operatorname{Id} & \text{otherwise.}
\end{cases}    
\end{equation}

\begin{equation}
C^0(\Delta(G), \widehat{\mathcal{F}}) = C^0(G, \mathcal{F}) = \underset{{v \in V(G)}}{\oplus}\mathcal{F}(v)
\end{equation}
implies $L^0_{\widehat{\mathcal{F}}},D^0_{\widehat{\mathcal{F}}} , L_{\mathcal{F}}, D_{\mathcal{F}}$ are linear maps from $C^0(G, \mathcal{F})$ to $C^0(G, \mathcal{F})$. When $v \in e$, there are two cofacets $[v,w]_e, [w,v]_e$ of $[v]_v$. Equations  \eqref{sheafLaplacianongraph}, \eqref{normalizedsheafLaplacianongraph},\eqref{upperdiagonalizedsheaflaplacian}, and \eqref{Laplacianformula} imply
\begin{equation}
L^0_{\widehat{\mathcal{F}}} = 2L_{\mathcal{F}}, D^0_{\widehat{\mathcal{F}}} = 2D_{\mathcal{F}}.    
\end{equation} 
Therefore,
\begin{equation}
\begin{split}
    \mathcal{L}^0_{\widehat{\mathcal{F}}} &=  (D^0_{\widehat{\mathcal{F}}})^{-\frac{1}{2}}L^0_{\widehat{\mathcal{F}}}(D^0_{\widehat{\mathcal{F}}})^{-\frac{1}{2}} \\
    & = (2D_{{\mathcal{F}}})^{-\frac{1}{2}}(2L_{{\mathcal{F}}})(2D_{{\mathcal{F}}})^{-\frac{1}{2}} \\
    &= (D_{{\mathcal{F}}})^{-\frac{1}{2}}(L_{{\mathcal{F}}})(D_{{\mathcal{F}}})^{-\frac{1}{2}} \\
    &= \mathcal{L}_{\mathcal{F}}
\end{split}
\end{equation}
and it finishes the proof.\hfill$\square$ \\
\end{proof}

Proposition \ref{sheaflaplacianongraphandassociatedsymsset} implies the normalized degree $k$ sheaf Laplacian on $\Delta(G)$ is a generalization of the normalized sheaf Laplacian on $G$.

\subsection{Neural sheaf diffusion on hypergraph}
In this subsection, we develop a degree $k$ NSD model on a symmetric simplicial set using the normalized degree $k$ sheaf Laplacian. By instantiating this framework on symmetric simplicial lifting of hypergraphs, we derive HNSD, enabling structure-aware learning on higher-order networks.

Let $X$ be a symmetric simplicial set, $\mathcal{F}$ be a cellular sheaf on $X$ of degree $k+1$ such that $\mathcal{F}(\sigma)=\mathbb{R}^d$ for any simplex $\sigma$. Let $\mathbf{x}_1, \cdots, \mathbf{x}_f$ be $k$-cochains and $\mathbf{X} \in \mathbb{R}^{|X_k|d \times f}$ be a matrix whose $i$th column is $\mathbf{x}_i$ for $1 \leq i \leq f$. Degree $k$ sheaf diffusion of $\mathbf{X}$ on $(X, \mathcal{F})$ is time dependent matrices $\mathbf{X}(t)$ describing the diffusion of $\mathbf{X}$ along $\mathcal{F}$ governed by
\begin{equation}\label{sheafdiffusion}
\mathbf{X}(0) = \mathbf{X}, \ \dot{\mathbf{X}}(t) = -\mathcal{L}^k_{\mathcal{F}}(\mathbf{X}(t)).   
\end{equation}
Equation \eqref{sheafdiffusion} is discretized via the explicit Euler scheme with unit step-size
\begin{equation}
\mathbf{X}_{t+1}=\mathbf{X}_t- \mathcal{L}^k_{\mathcal{F}}(\mathbf{X}_t).    
\end{equation}
A general layer of degree $k$ NSD on $X$ is defined by
\begin{equation}\label{NSDonsymsSet}
\mathbf{X}_{t+1}=\mathbf{X}_t - \alpha\left(\mathcal{L}^k_{\mathcal{F}(t)}(\mathbf{Id} \otimes \mathbf{W}^t_1)\mathbf{X}_t\mathbf{W}^t_2 \right)    
\end{equation}
where
\begin{itemize}
    \item $t \in \mathbb{N}$ is layer
    \item $\alpha$ is nonlinear function
    \item $\mathbf{W}^1_t \in \mathbb{R}^{d \times d}, \mathbf{W}^2_t \in \mathbb{R}^{f_t \times f_{t+1}}$ are learnable matrices at layer $t$
    \item $\mathcal{F}(t)$ is cellular sheaf on $X$ at layer $t$.
\end{itemize}

The degree $k$ NSD on symmetric simplicial sets generalizes NSD on graphs~\cite{bodnar2022neural} through Proposition~\ref{sheaflaplacianongraphandassociatedsymsset}. For a hypergraph $H$ with symmetric simplicial lifting $\Delta(H)$, we learn the cellular sheaf $\mathcal{F}$ on $\Delta(H)$ of degree 1 via

\begin{equation}\label{sheaflearning}
\mathcal{F}([v]_v \prec [v,w]_e) = \operatorname{MLP}(x_v \parallel x_{[v,w]_e})
\end{equation}
where 
\begin{equation}
x_{[v,w]_e} = \alpha'\left( M \alpha \left( W^T \left[ \begin{array}{c} x_{v} \\ 1 \end{array} \right] \odot \left[ \begin{array}{c} x_{w} \\ 1 \end{array} \right] \right) \right)   
\end{equation}
with $\alpha = \operatorname{ReLU}$, $\alpha' = \tanh$~\cite{hua2022high}. \\

The HNSD architecture on $H$ is defined as the degree 0 NSD on $\Delta(H)$ with sheaf learning governed by~\eqref{sheaflearning}. This formulation preserves the theoretical guarantees of NSD while enabling structural learning on higher-order relations through the symmetric simplicial set framework.

\section{Numerical experiments}
\label{sec:experiments}

In this section, we present the experimental results and provide a comprehensive analysis of the performance achieved by the proposed method. To thoroughly assess its effectiveness, evaluations were conducted across five diverse real-world datasets.

\subsection{Experiment settings}
\subsubsection{Dataset}

In our experiments, we utilized five real-world datasets. The Cora and Citeseer datasets are citation networks, while the Cora-CA and DBLP-CA datasets are co-authorship networks~\cite{yadati2019hypergcn}. Citation datasets capture relationships between scientific publications based on citation patterns, where nodes represent individual papers and hyperedges link groups of papers that are co-cited by another publication. Similarly, the co-authorship datasets reflect relationships between papers based on shared authorship. Here, nodes represent individual papers, and hyperedges connect papers authored by the same researcher(s). Citation datasets form small hyperedges, with average sizes around 3 and maximum sizes limited to 5–26, resulting in shallow and less informative group contexts. In contrast, co-authorship datasets have larger hyperedges, averaging 4–5 nodes and reaching up to 43 or even 202, providing richer and broader relational information. Additionally, we used the Senate datasets~\cite{fowler2006legislative}, which represent legislative cosponsorship networks within the United States Congress. These datasets model the relationships between legislators and the bills they sponsor or cosponsor. In this case, nodes represent individual senators, while hyperedges denote legislative bills that link sponsors to their cosponsors. Each node contains attributes indicating the legislator's political party affiliation. Detailed characteristics of each dataset are summarized in~\autoref{tab:datasets}. 

\begin{table}[H]
\centering
\caption{Dataset statistics. CE homophily shows the homophily score~\cite{pei2020geom} computed from clique expansion, while Avg. HE size denotes the average hyperedge size.}

\label{tab:datasets}
\begin{tabular}{cccccc} 
\toprule
                    & Cora  & Citeseer & Cora-CA & DBLP-CA  & Senate  \\ 
\midrule
\# nodes       & 2708  & 3312     & 2708    & 41302   & 282     \\
\# hyperedges       & 1579  & 1079     & 1072    & 22363    & 315     \\
\# classes          & 7     & 6        & 7       & 6      & 2       \\
Avg. HE size & 3.03  & 3.20    & 4.28   & 4.45 & 17.17  \\
Max. HE size & 5  & 26    & 43   & 202 & 31  \\
CE Homophily        & 0.897 & 0.893    & 0.803   & 0.869  & 0.498   \\
\bottomrule
\end{tabular}
\end{table}

\subsubsection{Task and evaluation}

\begin{table*}[ht]
\centering
\caption{Node classification accuracy on five benchmark hypergraph datasets.}
\label{tab:results}
\resizebox{0.9\textwidth}{!}{
\begin{tabular}{c|ccccc|c} 
\toprule
            & \textbf{Cora}         & \textbf{Citeseer}     & \textbf{Cora-CA}      & \textbf{DBLP-CA}      & \textbf{Senate}        & \textbf{AVG.}   \\ 
\midrule
CEGCN       & 75.32 ± 1.69          & 71.43 ± 1.34          & 76.68 ± 1.30          & 87.19 ± 0.30          & 48.17 ± 3.68          & 71.74           \\
HNHN        & 76.36 ± 1.92          & 72.64 ± 1.57          & 77.19 ± 1.49          & 86.78 ± 0.29          & 50.85 ± 3.35          & 72.76           \\
LEGCN       & 72.23 ± 1.60          & 71.84 ± 1.17          & 72.23 ± 1.60          & 84.26 ± 0.40          & 73.24 ± 10.29         & 74.76           \\
HCoN & 51.77 ± 2.23          & 43.48 ± 1.12          & 72.37 ± 1.08         & 89.98 ± 0.26        & 46.28 ± 4.66         & 60.78           \\
HyperGCN    & 74.19 ± 1.41          & 69.42 ± 3.49          & 70.00 ± 3.74          & 86.78 ± 2.39          & 53.66 ± 6.35          & 70.78           \\
AllDeepSets & 76.88 ± 1.80          & 70.83 ± 1.63          & 81.97 ± 1.50          & 91.27 ± 0.27          & 48.17 ± 5.67          & 73.82           \\
SheafHyperGNN    & 77.80 ± 2.24          & 73.93 ± 1.06          & 81.65 ± 1.50          & 88.93 ± 0.66          & 74.65 ± 5.90          & 79.39           \\
\midrule
HNSD(OURS)  & 79.28 ± 0.82 & 74.40 ± 1.47 & 82.58 ± 1.15 & 89.85 ± 0.44 & 78.45 ± 5.87 & 80.91  \\
\bottomrule
\end{tabular}
}
\end{table*}

Our experiment focuses on node classification in hypergraphs. Each node in the hypergraph is assigned a label corresponding to one of several classes, depending on the characteristics of the dataset. The goal of our task is to predict the labels of unknown nodes based on the given hypergraph structure. To achieve this, we partition each dataset into training, validation, and test sets with a ratio of 0.5:0.25:0.25. The neural network is first trained using the training dataset. The validation dataset is used to monitor the model’s performance, and the final test accuracy is measured at the epoch where the validation error is minimized. The evaluation metric used in our experiment is accuracy, which represents the proportion of correctly classified nodes among all tested nodes. To ensure the reliability of our results, each experiment is conducted 10 times, and we report the average accuracy across these runs.

\subsubsection{Baseline}
In our evaluation, we compared our approach against several established baseline methods. CEGCN integrates the clique expansion (CE) ~\cite{agarwal2005beyond, zhou2006learning} technique with GCNs. It transforms hypergraphs into standard graphs by converting each hyperedge into a fully connected clique, upon which conventional GCNs are applied. HNHN~\cite{dong2020hnhn} is a hypergraph neural network that alternates between updating hyperedge and node representations using a normalization scheme, thereby effectively capturing higher-order dependencies inherent in hypergraph structures. LEGCN~\cite{yang2022semi} incorporates label embeddings into the graph convolutional framework by transforming hypergraphs into standard graphs where node features are augmented with label information, enhancing the representation of relational structures. HCoN~\cite{wu2022hypergraph} jointly learns node and hyperedge embeddings by aggregating information from both sides and employs a reconstruction loss to preserve the original hypergraph structure. HyperGCN~\cite{yadati2019hypergcn} extends GCNs to the hypergraph domain by introducing a hypergraph Laplacian and reducing hyperedges into weighted pairwise edges in a selective manner, enabling the model to retain essential high-order relationships. AllDeepSets~\cite{chien2022you} models hypergraph message passing through two permutation-invariant multiset functions, both instantiated via Deep Sets. These functions, implemented as multilayer perceptrons (MLPs) with summation-based aggregation, provide universal approximation capabilities for multiset functions. SheafHyperGNN~\cite{duta2023sheaf} builds upon sheaf theory to enrich graph neural networks by incorporating both local and global contexts through a hypergraph-based framework. It introduces a sheaf-theoretic linear diffusion mechanism to effectively model higher-order interactions within hypergraphs.

\subsubsection{Implementation Details}
All experiments were conducted on an NVIDIA GeForce RTX 3090 GPU with 24GB of memory and 10,496 CUDA cores. Due to the larger memory requirements of the DBLP CA dataset, experiments for this dataset utilized an NVIDIA RTX A6000 GPU equipped with 48GB of memory and 10,752 CUDA cores. Our proposed model and baseline methods were implemented using PyTorch and executed within a CUDA-enabled environment.

All models were optimized using the Adam optimizer, with a default initial learning rate of $0.001$ and a weight decay of $1\times10^{-5}$ unless otherwise noted. A learning rate scheduler, specifically PyTorch's \texttt{ReduceLROnPlateau}, was employed to halve the learning rate whenever the validation loss did not decrease for five consecutive epochs. Training continued for a maximum of 500 epochs, incorporating early stopping if the validation accuracy did not improve by at least $10^{-4}$ for 10 consecutive epochs. The parameters of the model corresponding to the epoch with the highest validation accuracy were retained for evaluation.

To determine optimal hyperparameters, we conducted a comprehensive grid search, systematically exploring learning rates \{0.01, 0.005, 0.001\}, weight decay values \{$10^{-4}$, $10^{-5}$, $10^{-6}$\}, dropout rates \{0.4, 0.5, 0.7\}, and hidden layer dimensions \{64, 128\}. For each dataset, the hyperparameters yielding the highest average validation accuracy over 10 random train/validation/test splits were selected. The optimal hyperparameters for the proposed HNSD model varied per dataset as follows: for Cora, learning rate of $0.001$, weight decay of $1\times10^{-5}$, hidden dimension of 128, and dropout rate of 0.5; for Citeseer, learning rate of $0.001$, weight decay of $1\times10^{-4}$, hidden dimension of 128, and dropout rate of 0.7; for Cora CA, learning rate of $0.001$, weight decay of $1\times10^{-5}$, hidden dimension of 128, and dropout rate of 0.5; for DBLP CA, learning rate of $0.001$, weight decay of $1\times10^{-5}$, hidden dimension of 128, and dropout rate of 0.4; and for Senate, learning rate of $0.01$, weight decay of $1\times10^{-6}$, hidden dimension of 128, and dropout rate of 0.4.

Baseline methods were trained using publicly available recommended hyperparameters. In instances where complete hyperparameter specifications were unavailable, we applied the same grid search methodology detailed above to maintain fairness and consistency in comparisons.

\subsection{Results and Analysis}

\begin{table*}[!]
\centering
\caption{Ablation study results.}
\label{tab:ablation}
\resizebox{0.8\textwidth}{!}{
\begin{tabular}{lcccc}
\toprule
\textbf{Setting} & \textbf{Cora} & \textbf{Citeseer} & \textbf{Cora-CA} & \textbf{DBLP-CA} \\
\midrule
No Sheaf Learning & 32.41 $\pm$ 1.68 & 26.79 $\pm$ 2.41 & 34.62 $\pm$ 1.80 & 35.33 $\pm$ 0.37 \\
No Sheaf Diffusion & 73.59 $\pm$ 1.55 & 71.92 $\pm$ 1.24 & 73.59 $\pm$ 1.55 & 84.21 $\pm$ 0.22 \\
Base HNSD & 78.21 $\pm$ 1.23 & 73.79 $\pm$ 1.49 & 81.57 $\pm$ 1.80 & 89.43 $\pm$ 0.66 \\
+ Dynamic Sheaf & 78.21 $\pm$ 0.82 & 73.78 $\pm$ 1.80 & 82.35 $\pm$ 1.11 & 89.87 $\pm$ 0.38 \\
+ Sheaf Left Projection & 78.66 $\pm$ 0.78 & 74.11 $\pm$ 1.48 & 82.04 $\pm$ 1.49 & 89.97 $\pm$ 0.37 \\
+ All (Final Model) & 79.28 $\pm$ 0.82 & 74.40 $\pm$ 1.47 & 82.58 $\pm$ 1.15 & 89.85 $\pm$ 0.44 \\
\bottomrule
\end{tabular}
}
\end{table*}

\subsubsection{Overall Performance Comparison}
Table~\ref{tab:results} presents the classification accuracy of the proposed HNSD model across five benchmark datasets—Cora, Citeseer, Cora-CA, DBLP-CA, and Senate—spanning citation, co-authorship, and legislative cosponsorship networks. HNSD is evaluated alongside several competitive baselines, with accuracy serving as the primary metric for assessing node classification performance.

On citation networks (Cora and Citeseer), HNSD achieves marginal yet consistent improvements of 1.48\% and 0.47\%, respectively, over the best-performing baselines. These results indicate that HNSD effectively captures fine-grained dependencies among co-cited nodes by leveraging its sheaf-based diffusion mechanism.

On the Cora-CA dataset, HNSD achieves the highest performance, surpassing AllDeepSets by 0.61\%. On DBLP-CA, while HNSD remains competitive, it lags slightly behind AllDeepSets by 1.42\%. Notably, HCoN demonstrates strong performance on DBLP-CA but underperforms on citation datasets. This contrast highlights inherent structural differences between co-authorship and citation hypergraphs: co-authorship hypergraphs often contain larger hyperedges, enabling more coherent relational modeling, whereas citation-based hypergraphs typically comprise smaller hyperedges, potentially introducing noise during message aggregation. These structural characteristics likely account for the marked performance degradation of HCoN on citation datasets.

The most pronounced improvement is reported on the Senate dataset, where HNSD exceeds SheafHyperGNN by 3.80\%. Given the heterogeneity of node attributes and the sparsity of connections in this dataset, this result highlights the model's robustness in challenging, real-world scenarios.

Overall, HNSD achieves the highest average accuracy of 80.91\%, outperforming SheafHyperGNN and AllDeepSets by margins of 1.52\% and 7.09\%, respectively. These consistent improvements across diverse domains validate the effectiveness of incorporating sheaf-theoretic structures with node-adaptive diffusion.

\subsubsection{Comparison with SheafHyperGNN}
A key innovation of HNSD lies in its rigorous mathematical grounding. While SheafHyperGNN introduces sheaf-based diffusion for hypergraph learning, our method advances this approach by constructing a sheaf Laplacian on hypergraphs formalized as symmetric simplicial sets. This abstraction offers a principled framework that more faithfully captures higher-order relationships.

This enhanced theoretical foundation directly translates into empirical gains. Across all five datasets, HNSD consistently outperforms SheafHyperGNN, including a significant margin of 3.80\% on the Senate dataset. These results underscore the practical benefits of our design, demonstrating that a formally structured sheaf-based framework yields measurable performance advantages.

\subsubsection{Strength in Heterophilic Hypergraphs} 
The Senate dataset is a prime example of a heterophilic hypergraph, where nodes connected by hyperedges typically exhibit highly dissimilar attributes. This low CE homophily value indicates that, unlike traditional graphs where connected nodes tend to share similar features, the relationships between connected nodes in the Senate dataset are more complex and varied. The dataset contains political entities with diverse attributes (e.g., political affiliations and voting behaviors), and the connections between them are sparse, making it challenging for conventional models that assume homophily.

In this context, HNSD excels by leveraging a sheaf-based diffusion mechanism that is particularly well-suited for heterophilic settings. HNSD dynamically learns the sheaf stalks through an MLP layer, enabling the model to incorporate high-order structural information into the node representations. This allows the model to capture intricate relationships between nodes and hyperedges, even when the nodes themselves are dissimilar in attributes. The restriction map within the sheaf framework further enhances information propagation, ensuring that even dissimilar nodes within the same hyperedge can influence each other effectively.

Thus, HNSD's ability to learn node-sensitive dynamics and model higher-order relationships provides a significant advantage in the Senate dataset. By integrating sheaf-theoretic principles with node-specific diffusion, HNSD outperforms existing methods in handling heterophilic interactions, leading to a 3.80\% performance boost over the second-best model, SheafHyperGNN. This demonstrates that HNSD is particularly adept at handling real-world hypergraphs with irregular topologies and heterogeneous attributes, making it a strong contender for complex datasets like the Senate dataset.

\begin{figure}[!]
    \centering
    \includegraphics[width=\linewidth]{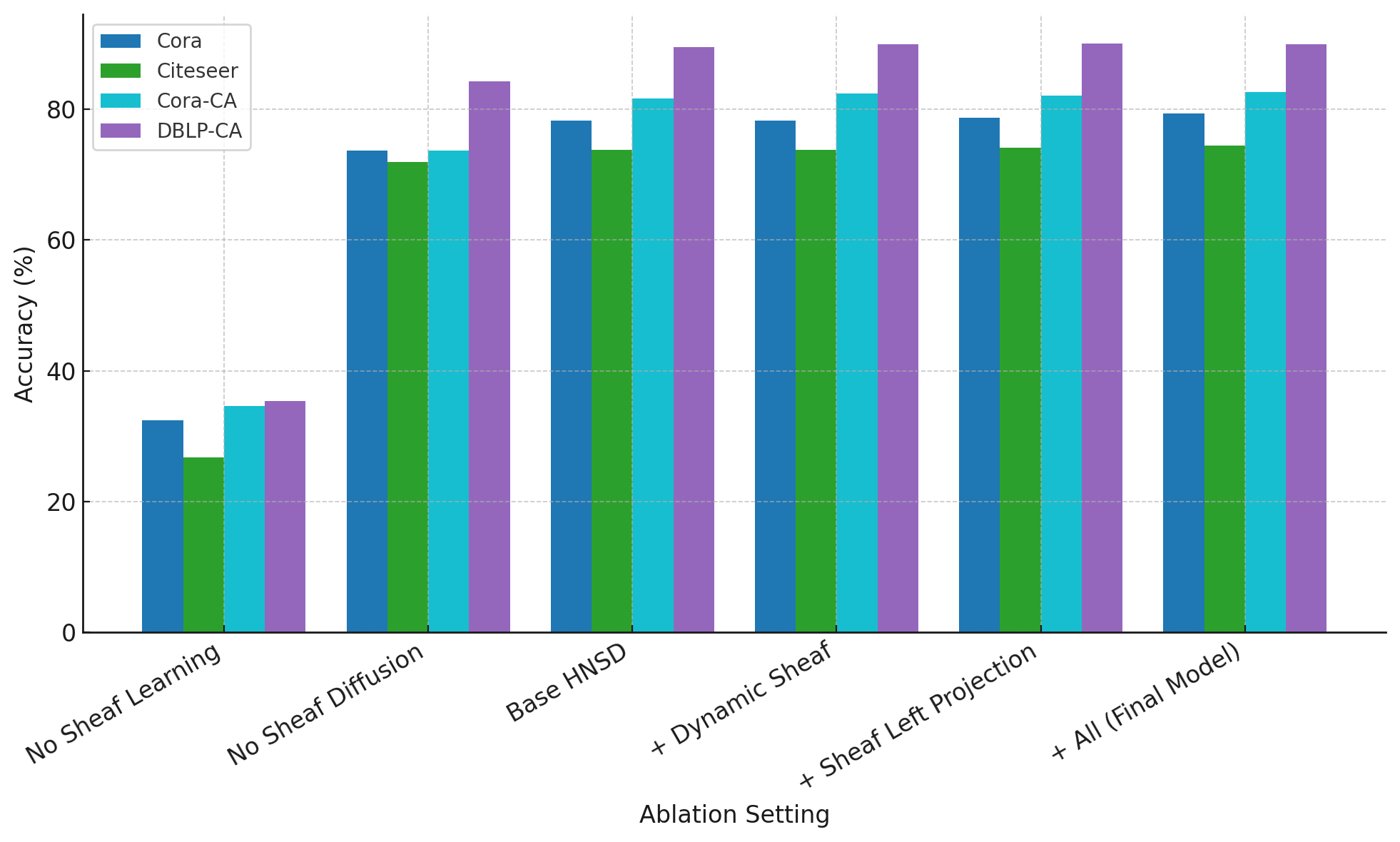}
    \caption{Ablation study results.}
\label{fig:tsneg}
\end{figure}

\begin{figure*}[!]
\centering
\begin{tabular}{ccc}

\begin{subfigure}[b]{0.23\textwidth}
    \includegraphics[width=\linewidth]{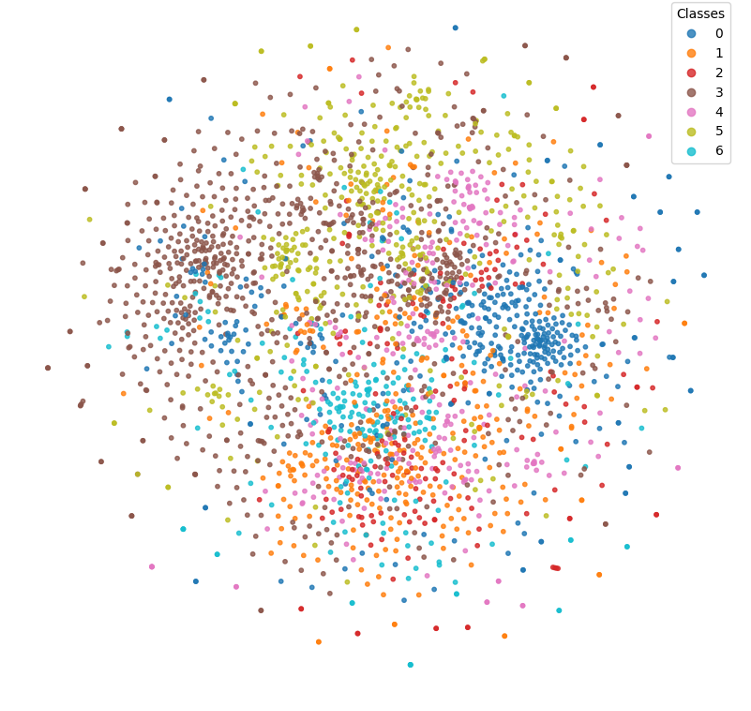}
    \caption{Cora - Raw.}
\end{subfigure} &
\begin{subfigure}[b]{0.23\textwidth}
    \includegraphics[width=\linewidth]{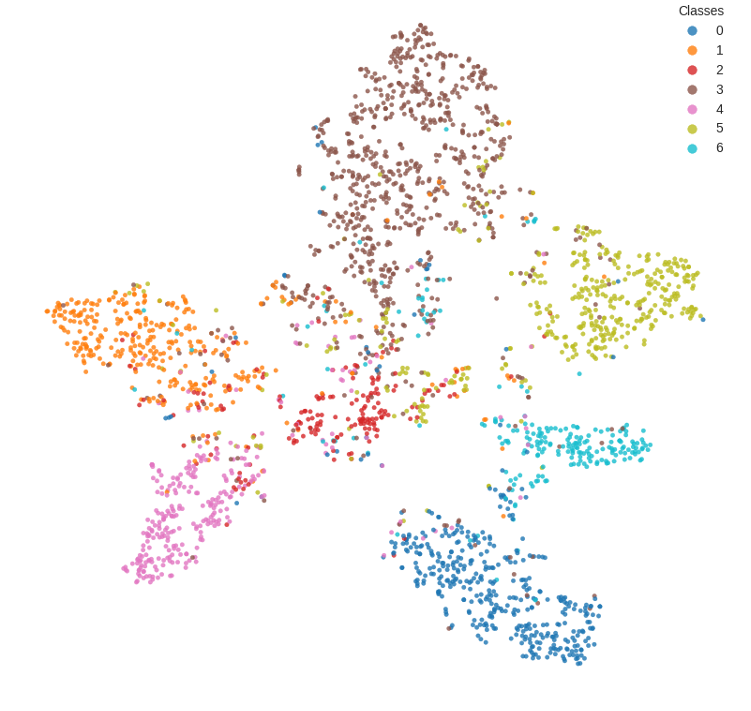}
    \caption{Cora - HyperGCN.}
\end{subfigure} &
\begin{subfigure}[b]{0.23\textwidth}
    \includegraphics[width=\linewidth]{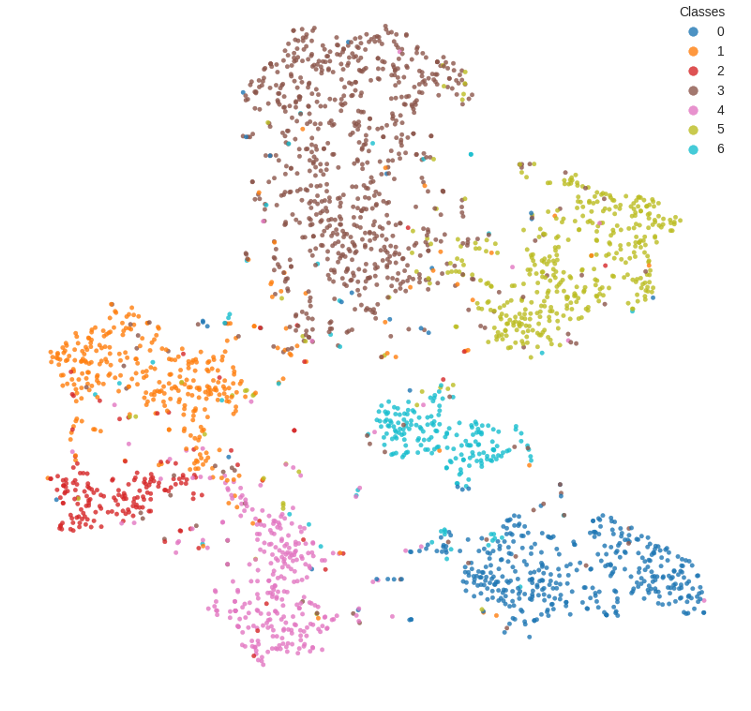}
    \caption{Cora - HNSD.}
\end{subfigure} \\[1ex]

\begin{subfigure}[b]{0.23\textwidth}
    \includegraphics[width=\linewidth]{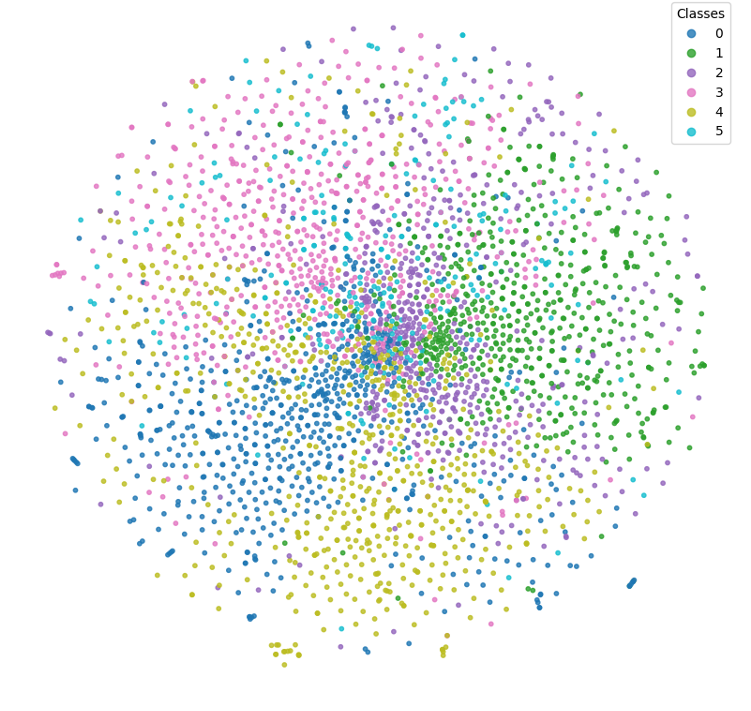}
    \caption{Citeseer - Raw.}
\end{subfigure} &
\begin{subfigure}[b]{0.23\textwidth}
    \includegraphics[width=\linewidth]{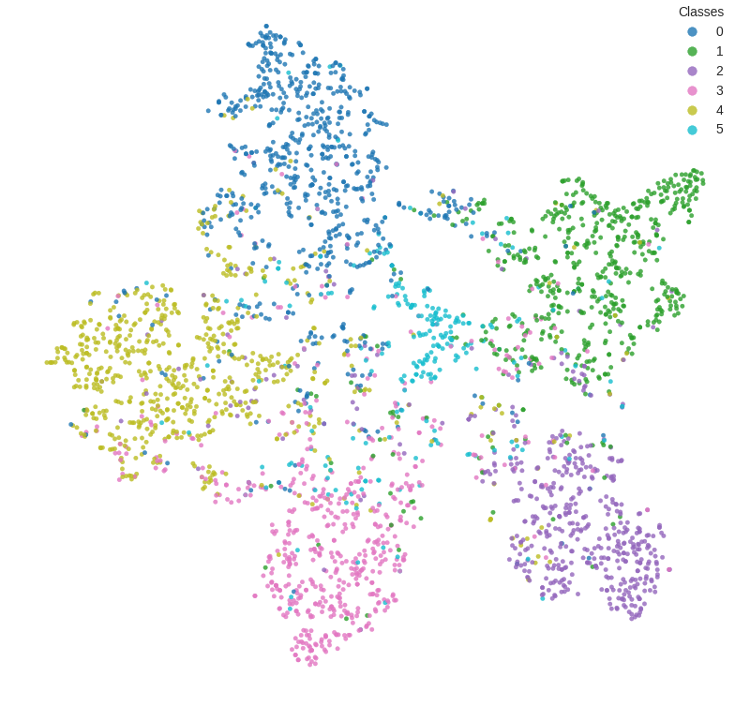}
    \caption{Citeseer - HyperGCN.}
\end{subfigure} &
\begin{subfigure}[b]{0.23\textwidth}
    \includegraphics[width=\linewidth]{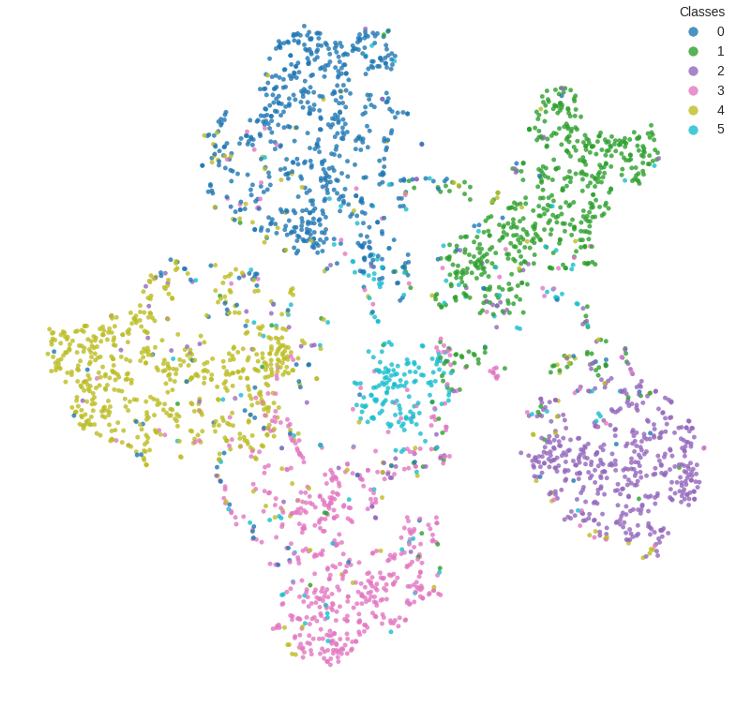}
    \caption{Citeseer - HNSD.}
\end{subfigure} \\[1ex]

\begin{subfigure}[b]{0.23\textwidth}
    \includegraphics[width=\linewidth]{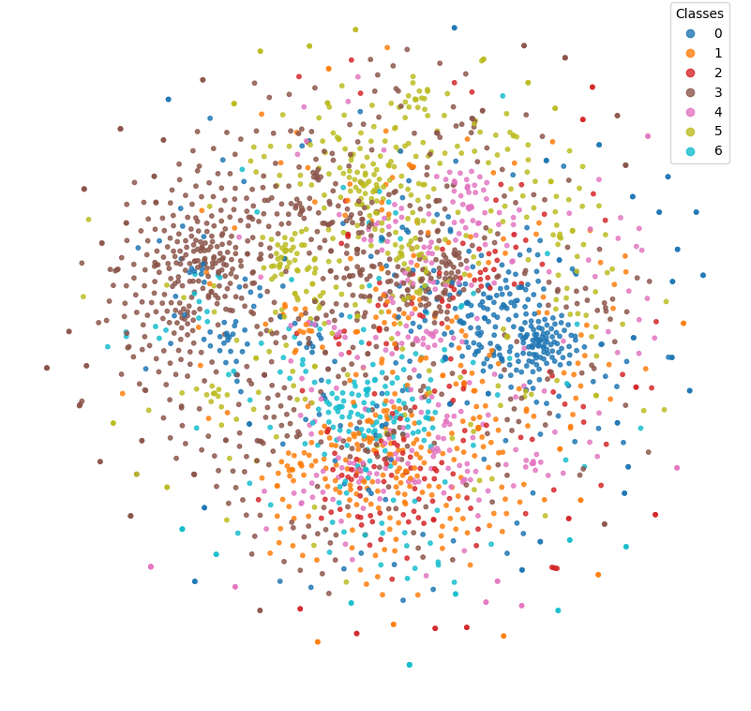}
    \caption{Cora-CA - Raw.}
\end{subfigure} &
\begin{subfigure}[b]{0.23\textwidth}
    \includegraphics[width=\linewidth]{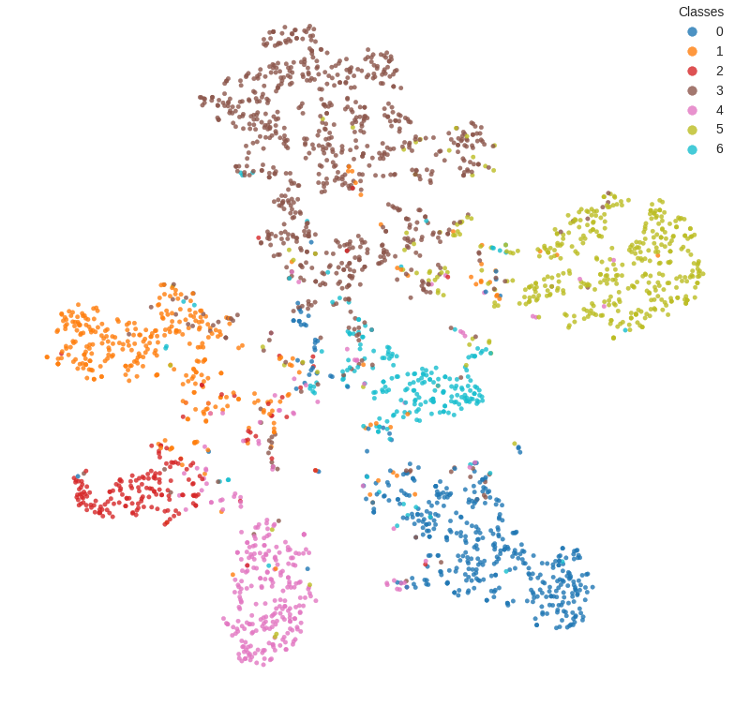}
    \caption{Cora-CA - HyperGCN.}
\end{subfigure} &
\begin{subfigure}[b]{0.23\textwidth}
    \includegraphics[width=\linewidth]{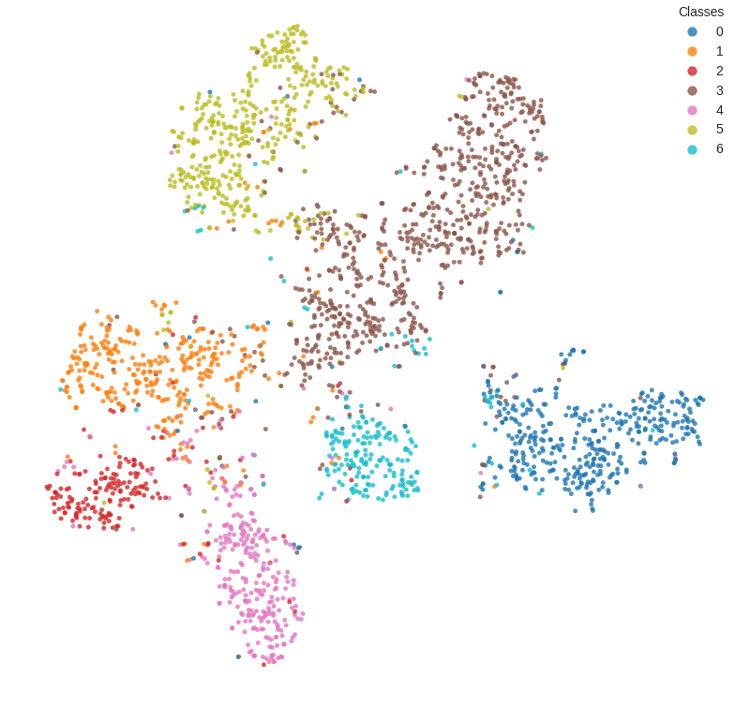}
    \caption{Cora-CA - HNSD.}
\end{subfigure} \\[1ex]

\begin{subfigure}[b]{0.23\textwidth}
    \includegraphics[width=\linewidth]{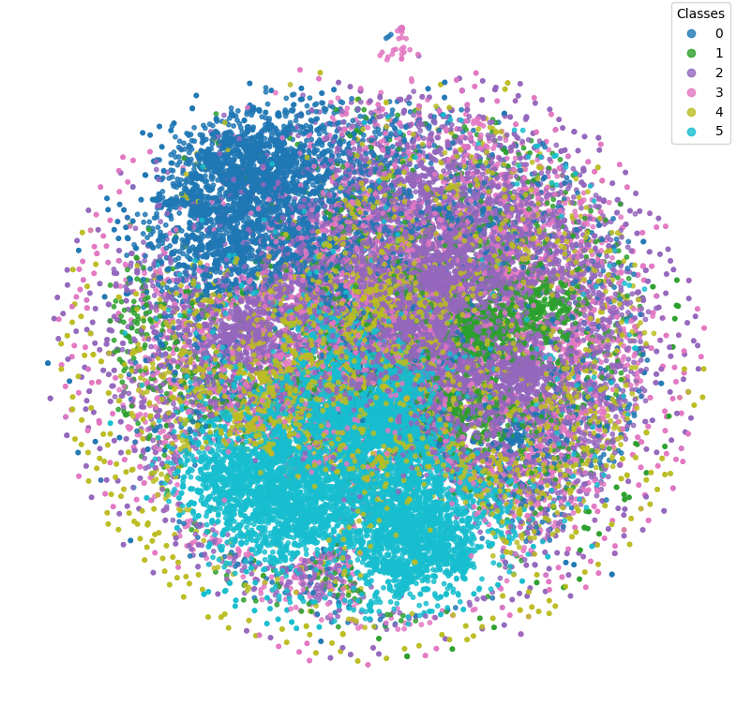}
    \caption{DBLP-CA - Raw.}
\end{subfigure} &
\begin{subfigure}[b]{0.23\textwidth}
    \includegraphics[width=\linewidth]{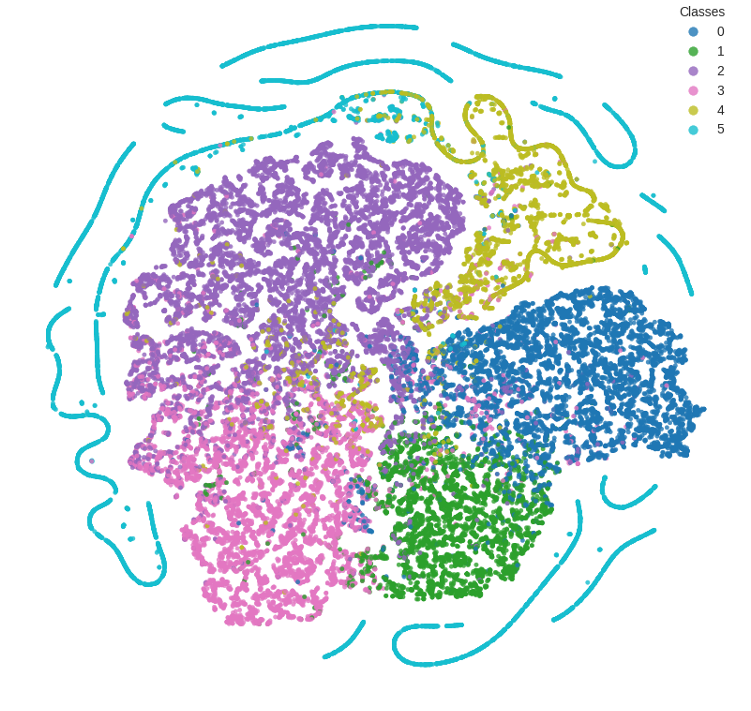}
    \caption{DBLP-CA - HyperGCN.}
\end{subfigure} &
\begin{subfigure}[b]{0.23\textwidth}
    \includegraphics[width=\linewidth]{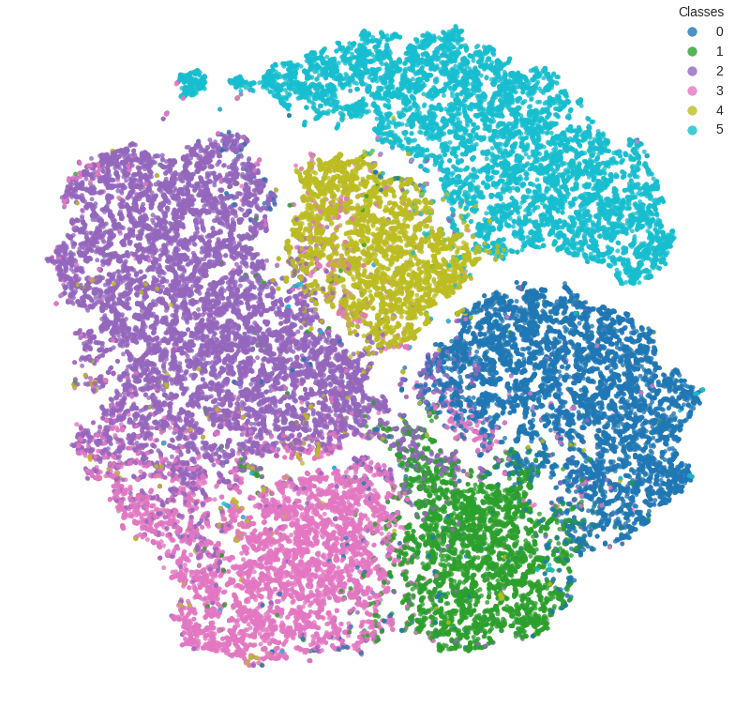}
    \caption{DBLP-CA - HNSD.}
\end{subfigure} \\[1ex]

\begin{subfigure}[b]{0.23\textwidth}
    \includegraphics[width=\linewidth]{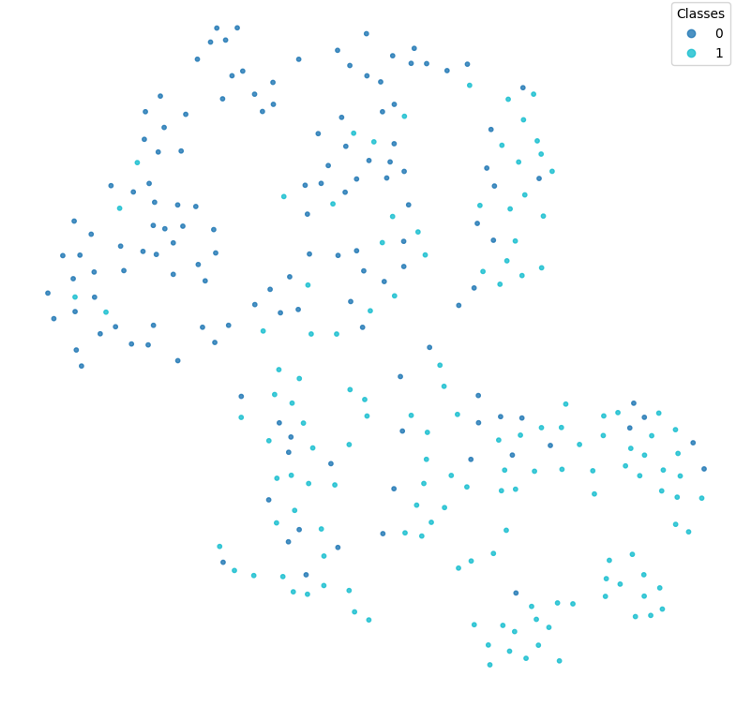}
    \caption{Senate - Raw.}
\end{subfigure} &
\begin{subfigure}[b]{0.23\textwidth}
    \includegraphics[width=\linewidth]{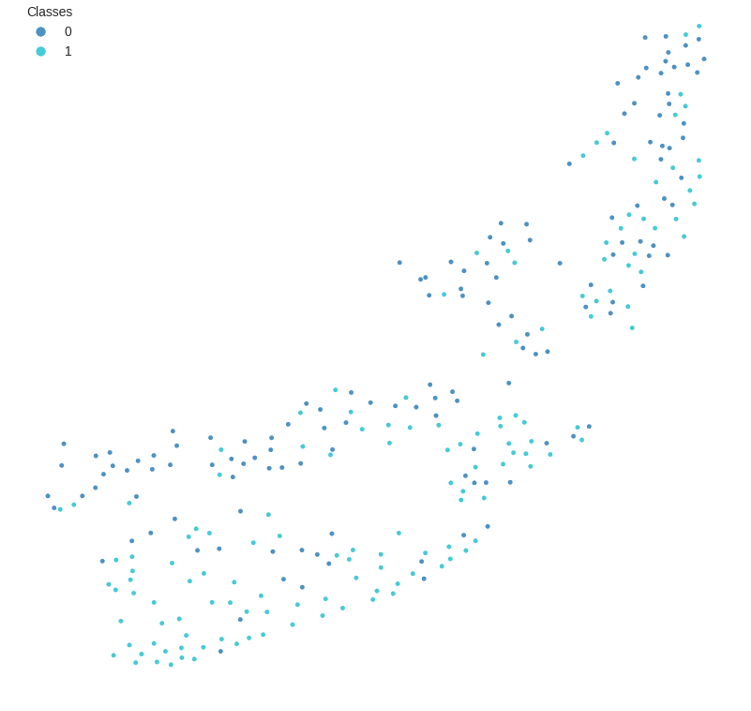}
    \caption{Senate - HyperGCN.}
\end{subfigure} &
\begin{subfigure}[b]{0.23\textwidth}
    \includegraphics[width=\linewidth]{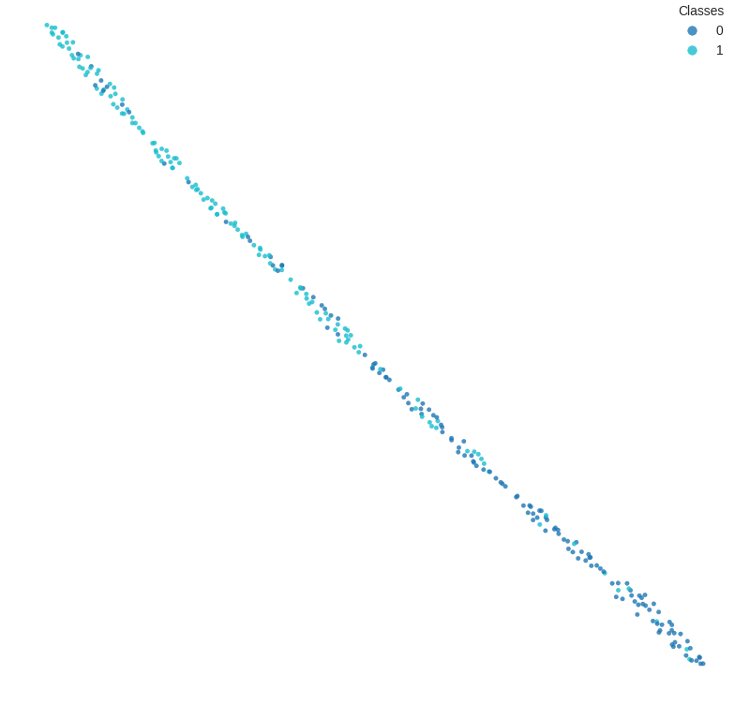}
    \caption{Senate - HNSD.}
\end{subfigure}

\end{tabular}
\caption{t-SNE visualizations of node embeddings across datasets and methods.}
\label{fig:tsne}
\end{figure*}

\subsubsection{Ablation Study}
To investigate the individual contributions of each component within our proposed HNSD framework, we conducted a systematic ablation study. Specifically, we incrementally modified or removed distinct architectural components of HNSD and assessed the resulting impacts on node classification accuracy.

Table~\ref{tab:ablation} and Fig.~\ref{fig:tsneg} summarize our findings. Removing the sheaf learning mechanism, thereby resorting to fixed, non-learnable sheaf maps, resulted in a significant decline in performance. This underscores the critical role of adaptive structural learning in effectively guiding the diffusion process. Conversely, the elimination of only the sheaf diffusion mechanism while maintaining the learned sheaf weights also caused notable performance degradation, albeit to a lesser extent.

Furthermore, we investigated the impact of two auxiliary design components: dynamic sheaf parameterization and left-side projection. The dynamic sheaf allows each layer to learn distinct sheaf parameters independently, leading to performance gains on several datasets. The left-side projection applies a linear transformation to node features prior to sheaf-based diffusion and consistently improves classification accuracy. Integrating both components yields optimal or near-optimal performance across most benchmarks.

Collectively, these results indicate that each component of the HNSD architecture contributes meaningfully to the overall performance. The combined application of all proposed enhancements yields consistently strong and accurate node classification results across a variety of hypergraph benchmarks.

\subsubsection{Feature Visualization}
To qualitatively evaluate the representational capability of HNSD, we visualize the learned embeddings using t-SNE across five different datasets, comparing them with embeddings derived from raw features and HyperGCN. As illustrated in Fig.~\ref{fig:tsne}, the embeddings generated by HNSD exhibit clearer and more distinct cluster formations compared to the baseline methods. This suggests that HNSD substantially enhances class discriminability. The advantage of HNSD is particularly evident in the Senate datasets, where traditional approaches typically encounter difficulties in distinguishing between classes due to inherent low homophily.

\section{Conclusion}
\label{sec:conclusion}
In this work, we introduced HNSD, a principled generalization of sheaf-based diffusion to hypergraphs. Our framework rigorously preserves the original hypergraph's structural information by constructing symmetric simplicial sets that encode all oriented subrelations. By leveraging adjacency defined via shared facets, the proposed normalized degree 0 sheaf Laplacian effectively captures non-trivial node interactions. We theoretically demonstrate that HNSD is a natural generalization of sheaf-based learning, as its Laplacian precisely reduces to the traditional normalized sheaf Laplacian on graphs. Empirically, HNSD achieves competitive performance across established hypergraph benchmarks, validating its efficacy as the first structure-preserving, sheaf-theoretic extension for hypergraph neural networks.

An interesting direction for future research involves extending HNSD to directed hypergraphs. This would entail utilizing a join operation to combine the symmetric simplicial sets for the tail and head of each hyperarc. For a general directed hypergraph, these unified structures could then be logically glued together. This construction would preserve the distinct adjacencies inherent to both tail and head components, enabling a new sheaf Laplacian to explicitly encode directionality for richer relational modeling.



\bibliographystyle{IEEEtran}
\bibliography{access}

\end{document}